\documentclass[10pt]{siamonline190516}
\usepackage{xcolor}
\usepackage{graphicx}    
\usepackage{tikz}
\usetikzlibrary{bayesnet}
\usetikzlibrary{arrows}
\usepackage{hyperref}
\usepackage{enumitem}
\usepackage{comment}
\usepackage{amssymb}
\theoremstyle{plain}
\renewtheorem{theorem}{Theorem}
\newtheorem*{theoremLoukas}{Theorem 4.2 in \cite{loukas2017close}}
\newtheorem*{corLoukas}{Corollary 4.2 in \cite{loukas2017close}}
\newtheorem*{cor41}{Corollary 4.1 in \cite{loukas2017close}}
\newtheorem*{eqn13}{Equation 1.3 in \cite{koltchinskii2017normal}}
\newtheorem*{eqn14}{Equation 1.4 in \cite{koltchinskii2017normal}}
\newtheorem*{eqn15}{Equation 1.5 in \cite{koltchinskii2017normal}}
\newtheorem*{cor113}{Corollary 1.12 in \cite{wei2017upper}}
\newtheorem*{theorem9}{Theorem 9 in \cite{koltchinskii2017concentration}}
\renewtheorem{lemma}[theorem]{Lemma}
\newtheorem{cor}[theorem]{Corollary}
\newtheorem{prop}[theorem]{Proposition}
\newtheorem{defn}[theorem]{Definition}

\newcommand{\PCA}{\mathrm{PCA}}
\newcommand{\OLS}{\mathrm{OLS}}
\newcommand{\norm}[2][]{{\left\Vert #2 \right\Vert}_{#1}}
\newcommand{\tr}[1]{#1^{\top}}

\newcommand{\given}{\,|\,}
\DeclareMathOperator*{\argmin}{argmin}

\usepackage{comment}

\newcommand{\R}{\mathbb{R}}

\definecolor{red1}{cmyk}{0,1,.8,0}
\definecolor{header1}{cmyk}{.9,.5,0,.35}
\definecolor{blue1}{cmyk}{.9,.7,0,0}
\definecolor{blue2}{cmyk}{.93,.95,.2,.07}

\renewcommand{\paragraph}[1]{\par\medskip\noindent\textbf{#1}}
\sloppypar
\newcommand\blfootnote[1]{%
  \begingroup
  \renewcommand\thefootnote{}\footnote{#1}%
  \addtocounter{footnote}{-1}%
  \endgroup
}

\usepackage{color}

\newcommand{\beginsupplement}{%
        \setcounter{table}{0}
        \renewcommand{\thetable}{S\arabic{table}}%
        \setcounter{figure}{0}
        \renewcommand{\thefigure}{S\arabic{figure}}%
     }

\title{ 
Dimensionality reduction, regularization, and generalization in overparameterized regressions 
}

\author{Ningyuan (Teresa) Huang\thanks{Department of Applied Mathematics and Statistics, Johns Hopkins University, and Mathematical Institute for Data Science, Johns Hopkins University.}
\and David W. Hogg\thanks{Flatiron Institute, a division of the Simons Foundation, and Center for Cosmology and Particle Physics, Department of Physics, New York University.}
\and Soledad Villar\footnotemark[1]}

\headers{
Dimensionality reduction and regularization
}{Huang, Hogg, \& Villar}

\begin{document}
\maketitle

\begin{abstract}
Overparameterization in deep learning is powerful: Very large models fit the training data perfectly and yet often generalize well. This realization brought back the study of linear models for regression, including ordinary least squares (OLS), which, like deep learning, shows a ``double-descent'' behavior: (1)~The risk (expected out-of-sample prediction error) can grow arbitrarily when the number of parameters $p$ approaches the number of samples $n$, and (2)~the risk decreases with $p$ for $p>n$, sometimes achieving a lower value than the lowest risk for $p<n$. The divergence of the risk for OLS can be avoided with regularization.  In this work, we show that for some data models it can also be avoided with a PCA-based dimensionality reduction (PCA-OLS, also known as principal component regression). We provide non-asymptotic bounds for the risk of PCA-OLS by considering the alignments of the population and empirical principal components. We show that dimensionality reduction improves robustness while OLS is arbitrarily susceptible to adversarial attacks, particularly in the overparameterized regime. We compare PCA-OLS theoretically and empirically with a wide range of projection-based methods, including random projections, partial least squares (PLS), and certain classes of linear two-layer neural networks. These comparisons are made for different data generation models to assess the sensitivity to signal-to-noise and the alignment of regression coefficients with the features. We find that methods in which the projection depends on the training data can outperform methods where the projections are chosen independently of the training data, even those with oracle knowledge of population quantities, another seemingly paradoxical phenomenon that has been identified previously.  This suggests that overparameterization may not be necessary for good generalization.
\end{abstract}

\blfootnote{Correspondence should be addressed to \textcolor{header1}{soledad.villar@jhu.edu}}

\section{Overparameterization and robustness in regression} \label{sec:intro}

One of the most remarkable properties of contemporary machine-learning methods---and especially deep learning---is that models with enormous capacity nonetheless generalize well. Overparameterized models have the flexibility to perfectly fit any training data, but (in many cases) still make good, non-trivial predictions on held-out or test data. Those good predictions contradict both our folklore and our intuitions. 

The realization that overparameterization is good for machine learning led to a reconsideration of classical linear regressions. It turns out that even linear regressions can generalize well in the overparameterized regime; that is, when the number of parameters $p$ far exceeds the number of training data points $n$ (provided that the fitting is performed in a min-norm or regularized way that limits the coefficients in the unconstrained $(p-n)$-dimensional subspace).
Both linear regressions and more complex machine-learning methods typically show a ``double-descent'' phenomenon, recently identified by Belkin et al. \cite{belkin2019reconciling}: (1) The underparameterized and overparameterized regimes are separated by a ``peaking'' phenomenon \cite{jain198239}, or ``jamming peak'' \cite{geiger2019jamming, spigler2018jamming}, in which the ``risk''---the expected out-of-sample prediction error---blows up when the model capacity just reaches overfitting (at $p=n$ in the linear case); (2) The risk further decreases with the number of parameters in the overparameterized regime, sometimes achieving a lower value than the underparameterized regime. 

The double-descent behavior raises the following research questions: Can we avoid the peaking phenomenon (RQ1)? Is overparameterization necessary for good generalization (RQ2)? The understanding of this phenomenon in the context of linear regression has provided some perspective on deep learning. An important result common to both deep learning and linear regression shows that regularization is important in or near the overparameterized regime \cite{hastie2020surprises, d2020triple, ju2020overfitting, xie2020weighted, li2020provable,nakkiran2020optimal, liu2020kernel, canatar2020spectral,wilson2020bayesian, hogg2021fitting}.

To address RQ1, we study double-descent in the context of linear regression, where the peaking phenomenon has a simple explanation that involves the (equivalent of the) condition number of the training features (the ratio of the largest singular value to the smallest). Prior work shows that the peaking phenomenon disappears with regularization using ridge penalty~\cite{hastie2020surprises}. In this work, we show that it also disappears with dimensionality reduction, another canonical form of regularization.

In some sense---and in the attitude we will take here---the peaking phenomenon at $p\approx n$ is a kind of lack-of-robustness. Similarly, susceptibility to adversarial attacks is also a kind of lack-of-robustness. These things ought to be related: In our view, a robust regression will not have divergent risk nor be extremely susceptible to attack. We connect these ideas here, and note that some models that generalize well nonetheless are extremely weak against adversarial attacks. In the overparameterized regime the default linear model (ordinary least squares) makes good predictions but is not robust in the sense that it is arbitrarily susceptible to attack.

Regressions have been attacked adversarially in a few ways. We focus here on attacks against the training features, but there are also attacks against the training labels \cite{Meilabelattack}, and attacks against the test data at test time.
We limit our discussion on the data-poisoning attack---that is, adding one adversarial data point to the training data \cite{biggio2012poisoning, li2020optimal}. We refer the readers to the recent works in \cite{Li_2020, li2020optimal, javanmard2020precise} for other forms of attacks including the Rank-1 (and Rank-k) attacks or the adversarial perturbation attacks. 

There are different ways to measure success for attacks against regressions, including increases in the risk \cite{javanmard2020precise},
distortion of the regression coefficients \cite{li2020optimal},
and other kinds of distortions to the properties of the data, e.g., \cite{advPCA}.
Here we are focused on robustness and prediction, so we care most about the risk.
Connected to our motivation and results, there is also adversarial training, which has been developed as a kind of regularization for regressions; it protects regressions from attack and also makes the peaking phenomenon disappear \cite{javanmard2020precise}. On the other hand, deep generative classifiers---that produce a generative model for the training data, similar to a low dimensional parameterization of the data distribution---are shown to be more robust against adversarial attacks than deep discriminative classifiers \cite{li2019generative, jebara2012machine}. We design a simple generative model for linear regression and demonstrate how it could act as an implicit regularization and avoids the peaking phenomenon. 

To investigate RQ2, we adopt the framework of projection-based methods, where the input data can be projected to a higher-dimensional feature space (i.e., overparameterization), or a lower-dimensional one (i.e., dimensionality reduction). This framework can be recast as a two-layer (linear) neural networks \cite{Ba2020Generalization}, where the first layer performs the projection (not trained), and the second layer performs linear regression (trained). We compare the risk behavior of multiple projection-based methods, including ordinary least squares after a PCA-based dimensionality reduction (PCA-OLS; also sometimes known as principal component regression), partial least squares, random projections, and classes of generative models and latent-variable models. All projection methods we consider herein involve transforming the input $X \in \R^{n \times p}$ linearly to some feature embeddings in $\R^{n \times k}$, followed by a linear regression on the transformed features. Since the regression takes the transformed features, the interpolation threshold (i.e., the peaking) appears at $k \approx n$ (instead of $p \approx n$). In this setting, overparameterization (i.e., $k > \min \{n,p\}$) is only possible when the projection matrix is independent of the training data (recall PCA-OLS is only possible for $k \le \min \{n,p\}$). Previous work has shown that the (individual) risk of data-independent projection methods monotonically decreases with $k$ when $k>n$ \cite{xu2019number, Ba2020Generalization, wu2020optimal}. However, it is not clear whether these overparameterized projection methods outperform their classic counterparts that choose the projection based on the training data, such as PCA-OLS and partial least squares.

We summarize our motivations and our contributions in Section~\ref{sec:condition}, after we give some problem setup and define some forms for linear regression in Section~\ref{sec:setup}. We follow that with analytical results in dimensionality reduction in Section~\ref{sec:PCAnopeak}, and a comparison with other projection-based regression models in Section~\ref{sec:comparison}. In Section~\ref{sec:attack} we discuss analytical results for adversarial attacks in the context of robustness of OLS in comparison with PCA-OLS, and in Section~\ref{sec:numerics} we provide numerical experiments.

\section{Linear regression: Problem setup and methods}\label{sec:setup}
Let $\{x_i,y_i\}_{i=1}^n$ where $x_i \in \mathbb R^{p}$ and $y_i\in \mathbb R$ for $i=1,\ldots, n$. We imagine having $n$ data points $(x,y)$ that (unknown to us) were generated from a joint Gaussian $\mathcal N(\mu, \Sigma)$ where $\mu=(\mu_x, \mu_y) = (0_p,0_1)$ and $\Sigma = \left( \begin{matrix} C_{xx} & C_{xy} \\ C_{yx} & C_{yy} \end{matrix} \right)$. In other words
    $p(x,y) = \mathcal N(\mu, \Sigma)$,
and therefore 
\begin{equation}
    \mathbb E_{y} [y\given x] = C_{yx} \, C_{xx}^{-1} (x - \mu_x) +\mu_y
\end{equation} (see for instance \cite{rasmussen2003gaussian} appendix A). In the case where $\mu=0$, this generative model is equivalent to  $x\sim \mathcal N(0, C_{xx})$ and $y= x^\top \beta + \epsilon$ where $\epsilon \sim \mathcal N(0, \sigma^2)$ and
\begin{equation}
    \beta:=  C_{xx}^{-1} \, C_{xy}, \quad 
    \sigma^2 := C_{yy} - C_{yx}\,C_{xx}^{-1}\,C_{xy} \, .
\end{equation}
This is now the standard linear generative model from the literature, with standard parameters $\beta$ and $\sigma$.
Let $X\in \mathbb R^{n\times p}$ and $Y\in  \R^{n\times 1}$ be training data in rectangular form.

We can consider several regression methods for finding an estimate $\hat \beta$ for the linear parameters $\beta$. Different estimators are compared in terms of their risk, that for our purposes will be defined as the expected squared error for out-of-sample predictions: 
\begin{align}
\mathcal R(\hat\beta) &= \mathbb E_{Y,x_\ast,y_\ast} [\| x_\ast^\top\hat\beta - y_\ast \|_2^2 \given X]
 \\
 &= \mathbb E_{Y,x_\ast} [\| x_\ast^\top(\hat\beta - \beta)||_2^2 \given X] + \sigma^2 \, ,
\end{align}
where $(x_\ast, y_\ast)$ are test points, fresh samples from the same distribution.

With this problem setup, there are many possible methods for performing linear regression:

\paragraph{Ordinary least squares (OLS):} 
This finds the linear combination of features $X$ that best predict the labels $Y$ in a least-square sense:
$\hat \beta_{\OLS} = \arg\min_{\beta} \| X\beta - Y\|^2_2$.
In the over-parameterized regime it chooses from among equivalent alternatives the min-norm solution.
We obtain $\hat \beta_{\OLS}$ by computing $X^\dagger \,Y$, where $X^\dagger$ denotes the pseudo-inverse of $X$ (the inverse that inverts only the non-zero eigenvalues of the matrix), namely:
\begin{align}
\hat \beta_{\OLS} = (X^\top X)^{-1} X^\top Y & \quad\text { if } p< n \label{eqn: OLS-under}\\
\hat \beta_{\OLS} = X^\top(X\, X^\top)^{-1} Y& \quad \text { if } p>n \, , \label{eqn: OLS-over}
\end{align}
assuming that $\operatorname{rank}(X)=\min\{p,n\}$.

\paragraph{Ridge regression:}
This is a variant of least-squares, but with an $l_2$ penalty on the regression coefficients, regularizing the fit:
$\hat \beta_{\lambda} = \arg\min_{\beta} \| X\beta - Y\|^2_2 + n\lambda\|\beta\|^2_2.$ 
There are other kinds of penalties by constraining the norm of the estimator, for instance $l_1$ (the Lasso), elastic net.

\paragraph{PCA-OLS:}
In this case we perform OLS, but---before starting---reduce the rank of the training data by performing a PCA-based dimensionality reduction:
\begin{equation}\hat \beta_{\PCA,k} = \arg\min_{\beta} \| X_{\PCA,k}\,\beta - Y\|^2_2 \, , 
\end{equation} where $X_{\PCA,k}$ is the rank-$k$ PCA approximation to $X$, with $k<\min{\{n,p\}}$. There are other equivalent formulations to PCA-OLS like the one in \cite{dhillon2013risk}. In our formulation,
\begin{equation}
\hat\beta_{\PCA,k} = (X_{\PCA,k}^\top X_{\PCA,k})^\dagger\,X_{\PCA,k}^\top\,Y.
\end{equation}
 
We remark that PCA-OLS is also known as Principal Component Regression (PCR) in the literature. In \cite{xu2019number}, Xu and Hsu studied the case where the true population covariance is known and the projection is onto the $k$ principal components of the population covariance. We shall call this regression method \textit{oracle}-PCR.

\paragraph{Partial least squares (PLS):}
This is a dimensionality-reduction based regression similar to PCA-OLS. PLS not only maximizes the variance of the projected features as PCA-OLS, but also the covariance of the projected responses and projected features. PLS is widely studied in the chemometrics and statistics literature \cite{wold1980model, helland1994comparison, cook2019partial}. The general form of PLS can be written as:
\begin{subequations}
\begin{align}
    X &\approx T \tr{P} ; \;  T \in \R^{n \times k}, P \in \R^{p \times k} . \\
    Y &\approx U \tr{Q} ; \;  U \in \R^{n \times k}, Q \in \R^{q \times k} .
\end{align}
\label{eqn:pls_gen}
\end{subequations}
When $Y$ is an univariate response variable, as in our analysis, PLS can be formulated as projecting the features to a Krylov space, followed by OLS \cite{rosipal2005pls}:
\begin{align}
  \hat{\beta}_{\text{PLS,k}} = \arg \min_{\beta} & \| \left(X \,  \Pi_{X_\text{PLS}} \right) \beta - Y\|^2_2 \, ,
\end{align}
where $\Pi_{X_{\text{PLS}}} = [s_{xy},  As_{xy}, A^2 s_{xy}, \cdots, A^{k-1} s_{xy}] , s_{xy} := \tr{X} Y, A :=  \tr{X} X$, and $[\cdot]$ denotes column concatenation. 
Note that OLS, ridge regression, PCA-OLS and PLS can be unified under the framework of continuum regression \cite{stone1990continuum, jong2001canonical}.

\paragraph{Random projection methods:} PCA and PLS project the original features $X$ via a \textit{data-dependent} projection matrix $\Pi \in \R^{p \times k}$ that is constructed from the training data. Other classes of methods of interest use random projections $\Pi$, chosen independently of the training data. All these projection methods can be written as 
\begin{equation}
    \hat{\beta}_{\Pi} = (X \Pi)^{\dagger} Y.
\end{equation}  
For example, the random orthogonal projection in \cite{lin2020causes} samples $\Pi$ uniformly over the set of orthogonal matrices such that $ \tr{\Pi} \Pi = I_k$ for $k \le p$ (or $ \Pi \, \tr{\Pi} = I_p$ for $k \ge p$); the random Gaussian projection in \cite{Ba2020Generalization} chooses $\Pi = [w_1, \cdots, w_k]$, where $w_i \stackrel{i.i.d}{\sim} \mathcal{N}(0, p^{-1} I_p)$. 

Data-dependent projections enforce the rank of $\Pi$ be at most $n$, while random projection allows $\operatorname{rank}(\Pi) = \min\{p, k\}$ to exceed $n$ if $p > n, k > n$. In this case, instead of reducing feature dimensions, random projection lifts the original features to a higher-dimensional space, which is key to kernel-based learning and deep neural networks. For example, \cite{Ba2020Generalization} identifies this model as a linear two-layer neural network where the first layer is random (untrained) and it performs a random Gaussian projection and the network is optimized only through the second layer parameters.  



\paragraph{Generative and latent-variable models:} In the classical machine learning literature, \emph{generative} approaches are those that attempt to learn the joint distribution of the observable variable $X$ and the target variable $Y$ \cite{ng2002discriminative}, or the distribution of $X$ conditioned on $Y$. We hint towards this approach earlier in Section \ref{sec:setup}, where we construct the data from a joint distribution $p(x,y)$ and we show how the generative formulation translates to the more common---discriminative---linear regression setting $Y=X\beta + \epsilon$. 

The generative regression model we propose---at training time---generates the features $X$ and the targets $Y$ as a linear function of latent variables $Z\in \mathbb R^{n\times k}$. In particular we aim to find $Z,P,Q$ such that $Y\approx ZP$, and $X\approx ZQ$, where $Z$ are the latent variables, and $P\in \mathbb R^{k\times1}$, $Q\in \mathbb R^{k\times p}$ are linear operators. 

Note that if $(Z, P, Q)$ is a generative model for $(X,Y)$, so is $(ZS, S^{-1}P, S^{-1}Q)$ for any $S$ invertible $k\times k$ matrix. Therefore we set $P$ to be a $k\times 1$ projection matrix given by the user, and we train the model to find $Z$ and $Q$.   


We define the generative linear regression as follows:
$\hat \beta_{generative} = {\hat{Q}^{\top\dagger}} P$
where $P$ is a $k\times 1$ projection matrix given by the user, $Q$ is a $k\times p$ operator, and $Z$ is a $n\times k$ matrix of latent variables. Estimates for $Q$ and $Z$ are found by
\begin{equation} \label{eq.generative}
\hat Q, \hat Z =\argmin_{Q,Z} \| X-  Z\, Q \|^2_2 \text{ subject to } Z\, P =Y.
\end{equation}
We train the model by solving \eqref{eq.generative} via alternately optimizing with respect to $Q$ and $Z$.

Our generative model can be viewed as fitting PLS when choosing the same projection matrix for both $X$ and $Y$ (i.e., $T = U \equiv Z$ in equation \eqref{eqn:pls_gen}). 
This is similar to the latent space model in \cite{hastie2020surprises}, where they relax the constraint to be $ZP \approx Y$. They provide an asymptotic risk analysis by simplifying $Q$ as an orthogonal projector and assuming the latent variables $Z$ from isotropic Gaussian distribution.
Our generative model also reduces to PCA-OLS if the constraints are removed. 

\section{Condition numbers, risk, and susceptibility to attack: Our contributions}\label{sec:condition}

Because $\hat{\beta}_{OLS}$ involves an inverse (or matrix solve or pseudo-inverse) of the $X^\top X$ or $X\,X^\top$ (which are related closely to the empirical variance of the features), the risk---the expected out-of-sample squared prediction error---will be  strongly dependent on the condition number of the empirical variance.
In general, the risk will get large as the condition number gets large. And indeed, the peaking phenomenon 
is related to the expectation of the condition number of this empirical variance \cite{hastie2020surprises}.

Ridge regression has been shown to avoid the peaking phenomenon \cite{hastie2020surprises}, and it does so by adding $n\,\lambda$ to the diagonal of the $X^\top X$ or $X\,X^\top$ matrix in the $\hat{\beta}_\lambda$ expression. This limits the condition number, makes the inverse well behaved, and limits the risk.

We conjecture that any regression method that controls or limits the condition number of the empirical covariance of the features will avoid or ameliorate the peaking phenomenon. This leads us to consider the PCA-OLS method, which replaces $X$ with a dimensionality-reduced copy of $X$, which thereby formally has infinite condition number, but in the context of a pseudo-inverse has a well-behaved effective condition number, so long as the $k$-th largest eigenvalue of empirical covariance matrix is well bounded away from $0$. (The effective condition number here is the ratio of the largest eigenvalue of the matrix to the smallest \emph{non-zero} eigenvalue.)  We conjecture that generative-model regressions (described above) will avoid the peaking for the same reasons. 
This also motivates us to analyze PCA-OLS under different data generating process and compare it with other projection-based methods, some of which do not control the condition number. 

In what follows (Section~\ref{sec:PCAnopeak}), we provide matching upper and lower bounds on the risk for PCA-OLS, in the setting where the ``effective rank" $r_0(C_{xx}) := \frac{\operatorname{tr}(C_{xx})}{\lambda_1} = o(n)$. The notion of effective rank is particularly useful in the analysis of overparameterized models in linear regression~\cite{bartlett2020benign} and principal component analysis \cite{koltchinskii2017normal}. It is also closely related to the study of basis expansion methods, which we further discuss in Section \ref{subsec:data_models}. Under the setting $r_0(C_{xx}) = o(n)$, we show that PCA-OLS remains bounded for all $k < n$ as long as the $k$-th largest population eigenvalue is bounded sufficiently away from $0$, whereas unregularized OLS further requires the population covariance $C_{xx}$ to have a heavy tail \cite{bartlett2020benign}, otherwise the risk of OLS blows up at $n \approx p$~\cite{hastie2020surprises}. This answers RQ1: Dimensionality reduction as a form of regularization can avoid the peaking phenomenon. We demonstrate our results in Section \ref{subsec:exp1}.

%


In Section \ref{sec:comparison}, we consider various projection-based methods and discuss their theoretical properties. Our analysis is supported by extensive experiments in Section \ref{subsec:exp2}. Using our framework introduced in Section \ref{sec:intro}, we vary the choice of projection dimension $k$ to compare the risk behaviors of these methods, where overparameterization occurs when $k>n$ (given $p > n$).  Remarkably, we empirically observe that projection methods independent of the training data, like oracle-PCR and random projections that can overparameterize, perform worse than projection methods based on the data, such as PCA-OLS. Although data-independent projection methods show decreasing risk with further overparameterization, in practice, they must be coupled with regularization to generalize well \cite{yang2020reduce, lin2020causes}  (and they do generalize well when regularized, see Section \ref{subsec:exp2}).  This answers RQ2: overparameterization is not necessary for good generalization. For example, PCA-OLS always perform dimensionality reduction where the optimal choice of $k < \min \{n,p\}$, and it seems to outperform all (unregularized) overparameterized methods.  

Superficially, our result is in contrast with that of by Xu and Hsu \cite{xu2019number}, who perform a principal component regression and observe a double-descent behavior. But in fact there is no contradiction: That prior result is based not on an empirical PCA of the features; it is based on an oracle version of principal component regression (oracle-PCR), a method that (unrealistically) requires knowledge of the true (unobservable) generating distribution of the features, that is widely studied in the statistics literature \cite{massy1965principal, frank1993statistical} and amenable to exact analysis using the Marchenko-Pastur distribution. Although oracle-PCR can sometimes yield smaller risk at $k>n$ (e.g. for isotropic covariance model), we empirically observe that it is no better than min-norm OLS (which is PCA-OLS when $k \ge \min \{n,p\}$, see Figure \ref{fig:highsnr}). Moreover, with high probability, oracle-PCR suffers from the peaking phenomenon at $k \to n$ due to unbounded variance \cite{xu2019number, wu2020optimal}, where PCA-OLS has a bounded variance given that the $k$-th largest population eigenvalue $\lambda_k$ is bounded away from $0$. The fact that PCA-based estimates are more robust than ones derived with oracle-PCR (at least in the regime $n\approx k$) seems to be an instance where using the predictor of the covariance decreases the variance of estimators with respect to using the true value. This sort of paradoxical phenomenon has been identified in different contexts within the statistics literature \cite{henmi2004paradox, tarpey2014paradoxical}.

Going beyond benign training data, we consider ``data-poisoning'' attacks in Section~\ref{sec:attack}, in which the attacker is permitted to add one data point $(x_0, y_0)$ to the training data prior to training. We find that if that data point is carefully chosen to increase substantially the condition number of the empirical covariance of the features (by, say, introducing a small but non-zero singular value), it also has a significant effect on the risk. We show that, in the OLS case, in the overparameterized regime $p>n$, the attack can be arbitrarily harmful to the risk, because it can arbitrarily increase the condition number of the empirical variance.  This implies that overparameterized projection methods are also susceptible to attacks, unless they are properly regularized.  We show that PCA-OLS and ridge regression are not nearly as susceptible to data-poisoning attacks, as expected, since they control the condition number of the matrix being inverted (or effective condition number of the matrix being pseudo-inverted). We conjecture that other kinds of regularized regressions, and the generative model defined in the previous section, will also be (at least partially) protected against such attacks. We show empirical evidence of such claims in Section \ref{subsec:exp1}.

\section{Generalization properties of PCA-OLS}\label{sec:PCAnopeak}

In this Section we analyze the risk of the estimator $\hat\beta_{\PCA,k}$ and we provide an upper bound: The risk is independent of the number of parameters $p$ and monotonically decreases with the number of samples $n$ while number of principal components $k$ is fixed. 

We define the expected value of an estimator $\hat \beta$ of the form of \eqref{eqn: OLS-under} or \eqref{eqn: OLS-over} conditioned on the training data as
\begin{equation}
    \tilde \beta := \mathbb E_{Y} [\hat\beta \given X] = \Pi_X\,C_{xx}^{-1}\,C_{xy} = \Pi_X\,\beta \, ,
\end{equation} 
where $\Pi_X$ is the $p\times p$ orthogonal projection onto the span of $X$. We note that when $\Pi_X$ is the identity or when the span of $X$ contains the span of $\beta$, the estimator $\hat \beta$ is unbiased.
Let $\Pi_{X_\perp}$ be the projection to the space orthogonal to the span of $X$.

Following the notation from \cite{hastie2020surprises}, the risk for the OLS case (where $\hat\beta = \hat\beta_{\text{OLS}}$) can be decomposed as a quadratic sum of bias plus variance in the following way:
\begin{align}
    \mathcal R(\hat\beta\given X) &= \underbrace{\mathbb E_{x_\ast} [(x_\ast^\top(\beta - \tilde\beta))^2 \given X] }_{\text{bias squared}}
    + \underbrace{\mathbb E_{Y,x_\ast} [(x_\ast^\top(\hat\beta - \tilde \beta))^2 \given X]}_{\text{variance}} + \sigma^2
    \\
    &= \underbrace{\beta^\top\,\Pi_{X_\perp}\,C_{xx}\,\Pi_{X_\perp}\,\beta}_{\text{bias squared}} 
      + \underbrace{\frac{\sigma^2}{n} \operatorname{tr}\left((\frac{1}{n}\,X^\top X)^\dagger\,C_{xx}\right)}_{\text{variance}}
    + \sigma^2, \label{eqn:risk_ols}
\end{align}

Hastie et. al. \cite{hastie2020surprises} give an expectation for the risk as a function of $p/n$ (in the limit of $n\to\infty$) making use of the Marchenko--Pastur distribution for eigenvalues of a random matrix. The key step of their argument writes $X = Z\,C_{xx}^{1/2}$,
where $Z$ is a standard spherical Gaussian. With this substitution and the cyclical property of the trace, the variance term becomes independent of $C_{xx}$ and it reduces to the integral of the inverse of the singular values of $Z$ with respect to the Marchenko--Pastur measure. Unfortunately the transformation $X=ZC_{xx}^{1/2}$ does not interact well with the PCA projection, which prevents us to use the cyclic property of the trace to obtain a closed-form expression for the risk. However, we are able to provide non-asymptotic bounds.

In order to analyze the risk of the PCA-OLS estimator $\hat\beta = \hat\beta_{\PCA,k}$ we write
\begin{align}
    \tilde\beta &= \Pi_{X_{\text{PCA}}}\,\beta, 
    \\
    \mathcal R(\hat\beta\given X) &= \underbrace{\mathbb E_{x_\ast} [(x_\ast^\top(\beta - \tilde\beta))^2 \given X] }_{\text{bias squared}} 
   + \underbrace{\mathbb E_{Y,x_\ast} [(x_\ast^\top(\hat\beta - \tilde \beta))^2 \given X]}_{\text{variance}} + \sigma^2
    \\
    &= \underbrace{\beta^\top\,\Pi_{X_{\text{PCA}\perp}}\,C_{xx}\,\Pi_{X_{\text{PCA}\perp}}\,\beta}_{\text{bias squared}} + \underbrace{\frac{\sigma^2}{n} \operatorname{tr}\left((\frac{1}{n}\,X_\text{PCA}^\top X_\text{PCA})^\dagger\,C_{xx}\right)}_{\text{variance}}
    + \sigma^2, \label{eqn:pca_var}
\end{align}
where $X_\text{PCA}$, $\Pi_{X_\text{PCA}}$, and $\Pi_{X_{\text{PCA}\perp}}$ are the equivalents of their non-PCA counterparts for the rank-$k$ PCA approximation to $X$. The proof of this statement is straightforward and we include it in Supplementary \ref{supple:bv-decomp}.




In Theorem \ref{thm.risk} we provide a coarse upper bound for the risk of PCA-OLS. The proof of this bound uses generic random matrix tools and makes no assumption on $\beta$ nor the covariance matrix $C_{xx}$. In Theorem \ref{thm.risk2} we provide more refined upper and lower bounds for the risk that relies on the alignment of the top eigenspaces of empirical and population covariance matrices. Theorem \ref{thm.risk2} assumes the population covariance matrix $C_{xx}$ and empirical covariance $\frac{1}{n}X^\top X$ satisfy a certain set of spectral gap assumptions from \cite{koltchinskii2017normal} and \cite{loukas2017close}. These assumptions hold when the number of samples is large enough, and the gap between distinct eigenvalues of the population covariance is not too small. We give a complete discussion later on this Section. 


\begin{theorem}\label{thm.risk}
Let $x_i\sim \mathcal N(0_p,C_{xx})$ $i=1,\ldots, n$, and $$y_i = x_i^\top \beta + \epsilon$$ where $\epsilon\sim \mathcal N(0, \sigma^2)$. Let $c, t$ be some constants, $\lambda_1$ be the largest eigenvalue of $C_{xx}$, and $r_0(C_{xx}) := \frac{\operatorname{tr}(C_{xx})}{\lambda_1}$ be the effective rank. Let
$M = c \lambda_1 \max \Big\{\sqrt{\frac{r_0(C_{xx})}{n}},\frac{r_0(C_{xx})}{n}, \frac{t}{n} \Big\}$, and assume $M < \lambda_k$.
Then with probability greater than   $1 - e^{-t}$ we have
\begin{align}
\mathcal R (\hat \beta_{\text{PCA-OLS}-k} \mid X) &= \mathbb{B}+\mathbb{V} + \sigma^2, \\
\lambda_p \| \Pi_{X_{\text{PCA}\perp}} \beta\|^2 \leq\;  &\mathbb B \le \| \beta \|^2 \Big(M + \lambda_{k+1} \Big), \label{thm:bias}\\
 \frac{\sigma^2}{n} \frac{k \lambda_p}{ \lambda_1 + M } \leq \; &\mathbb V \leq \frac{\sigma^2}{n} \frac{k \lambda_1}{ \lambda_k - M} \label{thm:var},
\end{align}
where $\| \cdot \|$ is the $2$-norm for vectors, and $k$ denotes the rank-$k$ PCA with $k < \min{\{n,p\}}$.
\end{theorem}

The proof of Theorem \ref{thm.risk} is in Appendix \ref{app.proof1}. The variance bound uses Von Neumann's trace inequality, and it can be quite loose when the eigenvalues of $C_{xx}$ decay fast. However, the lower bound and upper bound match when $C_{xx}$ is the identity (the precise non-asymptotic concentration statement is in Lemma \ref{lem:variance}, Appendix \ref{app.proof1}). Yet in this case, the bias term can go unbounded when $p \to \infty$. The bias lower bound is trivial, but the estimator is trivially unbiased if $\beta$ is in the span of the data. A more refined lower bound can be computed if one takes into consideration the alignment between the eigenspaces of the data and the eigenspaces of the population covariance (see Theorem \ref{thm.risk3}). The bias and variance upper bounds are mainly based on Koltchinskii and Lounici's concentration inequality  \cite{koltchinskii2017concentration}, similar to Lemma 35 of \cite{bartlett2020benign}. Their theorem statements are reproduced in Appendix \ref{app.proof1} for convenience.  

The upper bound in Theorem \ref{thm.risk} is well-controlled if: 1) The effective rank $r_0(C_{xx}) = \frac{\operatorname{tr}(C_{xx})}{\lambda_1}$ grows slower than $n$ when $p$ increases, so $M \to 0$ as $n \to \infty$, which implies bounded bias and is necessary for bounded variance; 2) The $k$-th largest eigenvalue $\lambda_{k}$ is bounded away from zero and independent of $p$, and thus $\lambda_1/\lambda_k$ is well-conditioned, yielding a small variance. This dimensionless bound shows that PCA-OLS does not exhibit the peaking phenomenon under mild conditions on the population covariance structure. 

\paragraph{Tighter risk bounds using spectral gap assumptions} 

In order to use perturbation analysis to estimate the distance between empirical covariance eigenvectors and their corresponding population covariance eigenvectors, we require a minimum separation among the population covariance eigenvalues. A classical assumption considers a population covariance with possibly repeated eigenvalues, but it requires all empirical covariance eigenvalues corresponding to the same population eigenvalue to be tightly clustered. 
Let 
\begin{equation}
    E := C_{xx} - \frac{1}{n} \tr{X} X
\end{equation} be the difference between the empirical and population covariance (i.e., the perturbation). 
Let $\lambda_1\geq \ldots \geq \lambda_p$ be the eigenvalues of the population covariance matrix $C_{xx}$.
Let $\Delta_{r}=\left\{i: \sigma_{i}(C_{xx})=\lambda_{r}\right\}$ be the $r$-th eigenvalue cluster and $m_{r}:=\operatorname{card}\left(\Delta_{r}\right)$ be its multiplicity. 
 Define $g_{r}:=\lambda_{r}-\lambda_{r+1}>0, r \geq 1 .$ Let $\bar{g}_{r}$ be the spectral gap of eigenvalue $\lambda_r$, which is defined as:
\begin{equation*}
\bar{g}_{r} = \begin{cases}
g_1 &r=1\\
\min \left(g_{r-1}, g_{r}\right) &r \geq 2 \,.
\end{cases}
\end{equation*}
The assumption used in the analysis \cite{koltchinskii2017normal} asks that the perturbation $E$ is small, in the sense that $\|E\|_{\text{op}}<\frac{\bar{g}_{r}}{2},$ such that all the empirical eigenvalues $\tilde{\lambda}_{j}, j \in \Delta_{r}$ are covered by an interval
$$
\left(\lambda_{r}-\|E\|_{\text{op}}, \lambda_{r}+\|E\|_{\text{op}}\right) \subset\left(\lambda_{r}-\bar{g}_{r} / 2, \lambda_{r}+\bar{g}_{r} / 2\right)
$$
and the rest of the empirical eigenvalues are outside of the interval
$$
\left(\lambda_{r}-\left(\bar{g}_{r}-\|E\|_{\text{op}}\right), \lambda_{r}+\left(\bar{g}_{r}-\|E\|_{\text{op}}\right)\right) \supset\left[\lambda_{r}-\bar{g}_{r} / 2, \lambda_{r}+\bar{g}_{r} / 2\right].
$$
To correctly align the leading $k$ clusters of empirical eigenvalues to their population counterparts, the diameter of each population eigenvalue cluster must be strictly smaller than the distance between any two clusters. To this end, we assume
\begin{equation}\|E\|_{\text{op}}<\frac{1}{4} \min _{1 \leq r \leq k} \bar{g}_{r}, \label{assumption1}
\end{equation}
which is the assumption we will use to apply the concentration results from \cite{koltchinskii2017normal} to the top $k$ eigenspace. 
In addition, we require 
\begin{equation} \label{assumption2}
    \operatorname{sgn}\left(\lambda_{i}>\lambda_{j}\right) 2 \widetilde{\lambda}_{i}>\operatorname{sgn}\left(\lambda_{i}>\lambda_{j}\right)\left(\lambda_{i}+\lambda_{j}\right), \, \forall i \in \{1, \cdots, k\}, j \in \{1, \cdots, p\}, j \neq i,
\end{equation} which is the condition for the top $k$ eigenspace alignment results in \cite{loukas2017close}. 

We remark that the assumptions \eqref{assumption1} and \eqref{assumption2} hold when assuming the population spectral gap for the top $k$ eigenspaces, and the sample size $n$ are sufficiently large. From Theorem 9 in \cite{koltchinskii2017concentration},  there exists a constant $c$ such that for any constant $1 < t < n$, with probability at least $1 - e^{-t}$:
\begin{equation}
   \|E\|_{\text{op}} \le c \lambda_1 \max \Big\{\sqrt{\frac{r_0(C_{xx})}{n}},\frac{r_0(C_{xx})}{n}, \sqrt{\frac{t}{n}} \Big\}. \label{bound.kl}
\end{equation}
Therefore if $k$ is fixed, the assumptions may continue to hold even when $p$ is large, as long as the effective rank $r_0 (C_{xx}) := \frac{\operatorname{tr}(C_{xx})}{\lambda_1}$ is $o(n)$. Note that these assumption do not hold when $C_{xx}$ is the identity (i.e., isotropic case), nor for the spiked covariance model in \cite{johnstone2018pca} (where the top eigenvalue is $\lambda_1>1$ and the rest are all 1).

\begin{theorem} \label{thm.risk2}
Let $x_i\sim \mathcal N(0_p,C_{xx}),  \, i=1,\ldots, n$, satisfying spectral gap assumptions \eqref{assumption1}  and \eqref{assumption2}. Let $$y_i = x_i^\top \beta + \epsilon$$ where $\epsilon\sim \mathcal N(0, \sigma^2)$. Let $w_{i j} = \frac{\lambda_j}{\tilde{\lambda}_i}, k_{j}^{2}=\lambda_{j}\left(\lambda_{j}+\operatorname{tr}(C_{xx})\right)$, and $t$ be a constant. Then with probability greater than 
 $1 - \sum_{i=1}^k \sum^p_{j=1, j \neq i} \frac{4 w_{i j} k_{j}^{2}}{n t\left(\lambda_{i}-\lambda_{j}\right)^{2}}$
 we have 
\begin{align}
 \mathbb V \leq \frac{\sigma^2}{n} \sum_{i=1} ^k \left( \frac{\lambda_i}{\lambda_i + \|E\|_{\text{op}}}  + t \right) \label{eqn:var_tight_final2}.
\end{align}
Furthermore, if we assume $k$ is fixed and $n\gg k, n \to \infty$, with probability greater than any constant $a \in (0, 1)$ we have
\begin{equation}
 \mathbb V \geq
 \frac{\sigma^2}{n} \sum_{i=1} ^k  \left( \frac{\lambda_i}{\lambda_i - \|E\|_{\text{op}}}  - o(1/n) \right) .
 \end{equation}
\end{theorem}

The proof of Theorem \ref{thm.risk2} is in Appendix \ref{app.proof2}.
Finally, we produce a lower bound for the bias by treating $\beta$ as random, which provides an ``average-case" analysis.  

\begin{theorem} \label{thm.risk3}
Let $x_i\sim \mathcal N(0_p,C_{xx}), \, i=1,\ldots, n$, satisfying spectral gap assumption \eqref{assumption2}. Let $$y_i = x_i^\top \beta + \epsilon$$ where $\epsilon\sim \mathcal N(0, \sigma^2)$, and $\beta$ is randomly drawn from an isotropic distribution, where $\mathbb E_{\beta} [\beta] = 0, \, \mathbb E_{\beta} [\beta \, \tr{\beta}] = I$. Let $k_{j}^{2}=\lambda_{j}\left(\lambda_{j}+\operatorname{tr}(C_{xx})\right)$, and $t$ be a constant.
Then with probability at least $1 - \sum_{i=1}^k \sum^p_{j=1, j \neq i} \frac{4 \lambda_{j} k_{j}^{2}}{n t\left(\lambda_{i}-\lambda_{j}\right)^{2}}$ we have
\begin{equation}
  \mathbb E_{\beta} [\mathbb B]  \geq \sum_{i=k+1}^p \lambda_i - kt.
\end{equation}
\end{theorem}

The proof of Theorem \ref{thm.risk3} is in Appendix \ref{app.proof3}.
Theorem \ref{thm.risk2} shows that the variance of PCA-OLS depends on the spectral gap for the leading $k$ eigenvalues: larger spectral gap yields better control of the variance. Theorem \ref{thm.risk3} illustrates that choosing large $k$ decreases the bias on average; it also shows that the larger the spectral gap, the smaller the constant $t$ can be chosen, and thus the tighter the lower bound. As we shall see in Section \ref{subsec:oracle}, this risk bound has the same leading order term as oracle-PCR.

We remark that for a non-vacuous probability bound, we require $\lambda_i \gg \lambda_j \approx 0$ for $i=1, \cdots, k, \, j=,k+1, \cdots, p$ (i.e., in the gapped covariance model, or exponential decay model), such that the terms associated with $j = k+1, \cdots, p$ vanish:
\begin{align}
   \sum_{i=1}^k \sum_{j \neq i} \frac{4 \lambda_{j} k_{j}^{2}}{n t\left(\lambda_{i}-\lambda_{j}\right)^{2}}
   &\approx  \sum_{i=1}^k \sum_{j=1}^k \frac{4 \lambda_{j}^2 \, \left(\lambda_{j}+\operatorname{tr}(C_{xx})\right)}{n t\left(\lambda_{i}-\lambda_{j}\right)^{2}}. \label{eqn:prob_vac2}
\end{align}
Since $k$ is fixed, the numerator is dominated by $\operatorname{tr}(C_{xx})$, which is $o(n)$ given that $r_0(C_{xx}) = o(n)$ and $\lambda_1$ is fixed. Thus, \eqref{eqn:prob_vac2} tends to 0 (slowly) when $n \to \infty$. 

We note that \cite{wahl2019note} also provides a non-asymptotic upper bound for the risk of PCA-OLS using perturbation analysis. However, they rely on the results from the upper bounds on the excess risk of principal component analysis---the difference between using the empirical eigen-projectors and the population eigen-projectors. In addition to perturbation arguments, we use the notion of effective rank to characterize the data model, which is particularly useful in analyzing the overparameterized setting.

\section{Comparison of PCA-OLS with other projection-based methods} \label{sec:comparison}

\subsection{Oracle-PCR} \label{subsec:oracle}
Xu and Hsu \cite{xu2019number} analyzed the double-descent phenomenon of performing principal component regression using the population covariance matrix (instead of the empirical covariance as in PCA-OLS), which we call oracle-PCR. They considered $\beta$ as random to derive the asymptotic risk. Wu and Xu \cite{wu2020optimal} extended the analysis of oracle-PCR to more general setting of $\beta$ (i.e., either fixed or random) and its alignment with the eigenvalues of $C_{xx}$.

We sketch the risk for oracle-PCR, following equation \eqref{eqn:pca_var} by replacing the empirical principal components with their population counterparts. Since the data is generated from a Gaussian distribution, we can assume $C_{xx} = \operatorname{diag}(\lambda_1, \cdots, \lambda_p)$ without loss of generality. Thus, when $C_{xx}$ is diagonal with $\lambda_1 \ge \cdots \ge \lambda_p$, oracle-PCR essentially truncates the design matrix $X \in \R^{n \times p}$ to its first $k$ columns. In other words, $X \, \Pi_{\text{oracle-PCR}} = X_{[:k]} \in \R^{n \times k}$, where $X_{[:k]}$ denotes the submatrix of the first $k$ columns of $X$. Let $\hat{\beta}_{k} = X_{[:k]}^{\dagger} \, Y \in \R^{k}$ and $\vec{0} \in \R^{p-k}$. The estimator is given by:
\begin{equation}
    \hat{\beta}_{\text{oracle-PCR-k}} = [\hat{\beta}_{k}, \vec{0}] \in \R^{p}.
\end{equation}

Let $C_{xx, [k:]} \in \R^{(p-k) \times (p-k)}$ be the principal submatrix of $C_{xx}$ by deleting the first $k$ rows and columns. Let $X_{[:k]}$ be the submatrix of the first $k$ columns of $X$. Recall that $C_{xx} = \sum_{j=1}^p \lambda_j u_j \tr{u_j}$, and $\beta = C_{xx}^{-1} C_{xy} \equiv \sum_{j=1}^p b_j u_j$. 
Let the empirical eigenvalues of $\frac{1}{n}\,\tr{X_{[:k]}} X_{[:k]}$ be $\tilde{\lambda}_1' \ge \cdots \ge \tilde{\lambda}_n'$, and $\tilde{u}_i'$ be the empirical eigenvectors $\tilde{\lambda}_i'$. We have

\begin{align}
    \mathcal R(\hat\beta_{\text{oracle-PCR,k}}\given X) &= \beta^\top\, C_{xx, [k:]}\,\beta + \frac{\sigma^2}{n} \operatorname{tr}\left((\frac{1}{n}\,\tr{X_{[:k]}} X_{[:k]})^\dagger\,C_{xx}\right)
    + \sigma^2  \\
    &= \sum_{i=k+1}^p b_i^2 \lambda_i + \frac{\sigma^2}{n}  \operatorname{tr}\left( \left(\sum_{i=1}^k \frac{1}{\tilde{\lambda}_i'} \tilde{u}_i' \tr{\tilde{u}_i'} \right)  \left(\sum_{j=1}^p \lambda_j u_j \tr{u_j}  \right) \right) + \sigma^2  \label{eqn:oracle-pcr}.
\end{align}

Consider the model in Theorem \ref{thm.risk3} by treating $\beta$ as random where $\mathbb E[b_i^2]=1$. The bias of oracle-PCR (i.e., first term in \eqref{eqn:oracle-pcr}) has expected value $\sum_{i=k+1}^p \lambda_i$, similar to PCA-OLS (see Theorem \ref{thm.risk3}). The bias also illustrates the alignment between the coefficients of $\beta$ and the principal components: if they are misaligned in the sense that $b_i$ is large for $i = k+1, \cdots, n$, then the bias is large. 

Moreover, oracle-PCR tends to have larger variance than PCA-OLS. Indeed, the submatrix $X_{[:k]} \in \R^{n \times k}$ has smaller empirical singular values than the original data $X$, $\tilde{\lambda}_i' \leq \tilde{\lambda}_i$ for $i=1,\cdots,k$. This is a consequence of the interlacing theorem of singular values \cite{queiro1987interlacing}.
 
However, it is possible that the empirical eigenvectors of oracle-PCR (i.e. $\tilde{u}_i'$) have better alignment with the population eigenvectors (i.e. $u_i$) than those of PCA-OLS (i.e. $\tilde{u}_i$). Thus, it is not clear whether PCA-OLS always has a smaller variance. Nevertheless, when $k \approx n$, oracle-PCR works with an ill-conditioned $X_{[:k]}$, while PCA-OLS works with a rank-$k$ matrix $X_{\text{PCA,k}}$ that is well-conditioned, given that the effective rank $r_0 (C_{xx}) = o(n)$ and the $k$-th largest population eigenvalue is bounded away from $0$ (see Theorem \ref{thm.risk}). 

Remarkably, even when the true population covariance is known, use of the empirical covariance could yield better performance; we demonstrate this in Section \ref{subsec:exp2}.

\subsection{Other projection methods} \label{subsec:other_proj}

We briefly discuss a few other projection methods.

\paragraph{Parial Least Squares (PLS):} 
Similar to PCA-OLS, PLS is a data-dependent dimensionality reduction method. Helland and Almøy in \cite{helland1994comparison} compared the asymptotic performance of PCA-OLS versus PLS in the underparameterized regime where $p$ is fixed and $n \to \infty$ . They analyzed the alignment of data features and the regression coefficient via the notion of eigenvalue \textit{relevance}: an eigenvalue is irrelevant for the regression if it corresponds to the principal component that has small correlation with the dependent variable $Y$. They showed that PCA-OLS performs well when the irrelevant eigenvalues are extremely small or extremely large, while PLS does well for intermediate irrelevant eigenvalue. In the overparamaterized regime, Cook et al. analyzed PLS for various alignments of $n,p$ in the asymptotic regime \cite{cook2019partial}. They showed that PLS achieves its best performance in data models with many weak features (i.e., abundant regression).
 

\paragraph{Random orthogonal projection:} 
 Another feature extraction method is random orthogonal projection. The recent work \cite{lin2020causes} provides a very thorough analysis of the risk in this case, interpreting this model as a two-layer linear neural network, where the first layer is a random orthogonal projection (not trained) and the second layer performs ridge regression. 
 
In the isotropic case where $C_{xx}$ is the identity, random orthogonal projections are equivalent to performing oracle-PCR: Since all the population eigenvalues are the same, choosing the first $k$ principal components in oracle-PCR is equivalent to randomly select $k$ orthogonal directions in random orthogonal projection.

\paragraph{Random Gaussian projection:}

In \cite{Ba2020Generalization}, Ba et al. analyzed the asymptotic risk of random Gaussian projections in the isotropic covariance case, assuming that both the data $X$ and the projection matrix $\Pi$ are generated from a Gaussian distribution with zero mean and identity covariance. 
More quantitative comparisons with PCA-OLS can be found in Supplementary \ref{supple:rand_proj}.

\subsection{Data models with increasing number of features and the choice of \texorpdfstring{$k$}{k}} \label{subsec:data_models} 

Increasing number of input features can improve generalization for certain data models \cite{bartlett2020benign} and methods (such as PLS in abundant regression), but this is not always true. There are some explicit examples, like the one in Section 3 of \cite{jain2000statistical} where an increasing number of features deteriorates the performance. 

In practice, the number of features can be made arbitrarily large when fitting the original input features with a flexible functional form, such as a data-independent projection method, or a basis-function expansion (for example, a polynomial or a Fourier series, like the model analyzed in \cite{xie2020weighted} and \cite{hogg2021fitting}). In the basis-function expansion setting, each data point $y_i$ is associated with a (scalar or vector) location $t_i$ in some ambient space $\R^p$, and the new features $[X]_{ij}$ are created by basis-function evaluations $g_j(t_i)$. The expansion of the locations $t_i$ into the feature matrix $X \in \R^{n \times k}$ can be viewed as a feature embedding. Hence, the number of functions $g_j(t)$ represents the number of features $k$; if the basis is infinite, $k$ can be made arbitrarily large. Another popular non-parametric method, Gaussian Processes (GPs), can also have infinite number of parameters; any GP can be seen as the limit as $k\rightarrow\infty$ of a specifically created and regularized least-squares, with Mercer's Theorem and the kernel trick converting the $k$-sums in the outer product $X\,X^\top$ into kernel-function evaluations \cite{li2021generalization,hogg2021fitting}.

Choosing $k$ is an active area of data science research: it appears in the context of projection-based methods, basis-function expansion, and the architectural design of neural networks. In practice, it is typically addressed using cross-validation. The choice of $k$ in PCA-OLS has been studied extensively: from a feature selection perspective, \cite{jolliffe1986principal} reviewed some iterative procedures to select variables; from a model selection perspective, \cite{hoyle2008automatic} discussed an improved Bayesian model evidence criteria to select the number of components in a high-dimensional, small sample size setting.

\section{Adversarial attacks in linear regression}\label{sec:attack}

Adversarial attacks are a very interesting phenomenon that was discovered in the context of deep learning, where imperceptible perturbations of the data can produce huge changes in the predictions of classifiers \cite{szegedy2014intriguing}.  
 
Recent work proves that adversarial examples are ubiquitous for classification in the overparameterized regime  \cite{belkin2018overfitting}: Assuming non-zero label noise, the set of adversarial examples is asymptotically dense in the support of data, so there is an adversarial example arbitrarily close to any data point. 

In the context of regressions, many kinds of adversarial attacks had been studied as well  \cite{biggio2012poisoning, Li_2020, li2020optimal,javanmard2020precise,advPCA}. 
 In the robust regression literature, most studies rely on the assumption that the feature vector $\beta$ is sparse \cite{xu2012robust, chen2013robust, robust_linear_regression}. For example, \cite{chen2013robust} showed that data poisoning attacks can completely change the support of $\beta$ in sparse regression. Here, we provide straightforward analysis that does not require sparsity assumptions, which is naturally hinted in our discussion of data models with large number of weak features (see Section \ref{subsec:data_models}).

In what follows, we consider only a ``data-poisoning'' attack, in which the attacker adds a single data point $(x_0, y_0)$ to the training data with the goal of having it (dramatically) increase the risk on test data.

\begin{defn}(Data-poisoning attack).
Let $(X,Y)$ be the original training data. Let us consider an additional adversarial pair $(x_0,y_0)$, and let:
$$
\tilde{X} = 
\begin{bmatrix}
X \\
\tr{x}_0
\end{bmatrix} \in \mathbb R^{(n+1) \times p},
\quad
\tilde{Y} = 
\begin{bmatrix}
Y \\
y_0
\end{bmatrix}
\in \mathbb R^{n+1}.
$$
The attacker's goal is to maximize the empirical risk subject to the constraints on both $x_0$ and $y_0$ (note that $x,y$ might be in different units, so we shall apply the constraint separately).
\begin{equation}
 \max_{\norm{x_0} \le \epsilon, \norm{y_0} \le \epsilon} \| \tilde{Y} - \tilde{X}\beta \|^2 .  \label{eqn:poison_goal}
\end{equation}
\end{defn}

We find that, for ordinary (unregularized) least squares, in the overparameterized regime, data-poisoning attacks can be arbitrarily successful. The fundamental reason is that the new $p$-dimensional data point $x_0$ can lie very near (but not precisely in) the $n$-dimensional subspace spanned by the extant data $X$; it then increases dramatically the condition number of the empirical covariance $X^\top\,X$ and makes the ordinary least-squares regression arbitrarily sensitive to the training labels $Y$ and $y_0$. This is in contrast to the underparameterized regime, in which the effect of the attack is limited \cite{li2020optimal}. The following propositions are very simple. Their proofs are in the Supplementary Material.

\begin{prop} \label{prop:attack}
In the overparameterized regime, ordinary least squares is arbitrarily sensitive to data-poisoning attacks. The risk tends to infinity when the additional adversary feature $x_0$ is arbitrarily close to the column space of $X$. In other words, let $\tilde{X}= \tr{[\tr{X}; x_0]}$. If $x_0 = \sum_{i=1}^n \alpha_i v_i + \delta $ where $\operatorname{Col}(X) = \operatorname{span}\{v_1,\cdots,v_n\}$, then $\mathcal R(\hat\beta\given \tilde{X}) \to \infty$ when $\delta\to 0$.
\end{prop}

In contrast, to make successful data-poisoning attack in the underparameterized regime, $\tr{X} X$ has to be ill-conditioned.

\begin{prop} \label{prop:ols-under}
In the underparameterized regime, ordinary least squares is arbitrarily sensitive to data-poisoning attack if the smallest eigenvalue of $\tr{X} X$ is smaller than the attack strength $\epsilon$.
\end{prop}

\begin{cor}\label{cor:pca-attack}
If the $k$-th largest eigenvalue of $\tr{X} X$ is much greater than $\epsilon$, then PCA-OLS is robust to data-poisoning attack.
\end{cor}

Proposition \ref{prop:ols-under} and Corollary \ref{cor:pca-attack} show that PCA-OLS are robust to data-poisoning attack in the natural PCA setting where the first $k$ eigenvalues have high energy. In other words, data-poisoning attack fails if the empirical covariance of the features is well-conditioned. The same reasoning tells us that ridge regression will be robust to this specific attacks.

We conjecture, and our experiments suggest, that this particular attack will also become bounded for the generative model.  

We emphasise that the attacks described in this Section are specifically tailored to exploit the vulnerability of unregularized ordinary least squares. It remains an open problem to study attacks tailored to the other regressions that maximizes the risk. For instance, one could presume an attack for PCA-OLS that imposes a structure in which $y$ is predicted by the low variance components of $x$.  
Nevertheless, PCA-OLS can easily detect such attack with an outlier-based defense strategy: If the attack aims to change the $k$-th largest empirical eigenvalue, then the adversarial pair must lie outside the rank-$k$ principal component subspace, and the magnitude of the attack must grow with $n$.
Thus, the poison point will appear as an outlier compared to the original data features, given that $n$ is large and the population covariance has most energies in the first $k$ components.

Our results also suggest new insights of using PCA as a data preprocessing to defend adversarial attacks in classification settings \cite{bhagoji_pca, carlini2017adversarial, alemany2020dilemma}.
\cite{alemany2020dilemma} showed that the effectiveness of PCA defense depends on choosing the correct number of principal components, in the sense that it matches the intrinsic dimension of the data manifold. This is transparent in our Corollary \ref{cor:pca-attack}: If PCA-OLS chooses a wrong $k$ that corresponds to a low variance components, the attack becomes considerably powerful.






\section{Numerical experiments}\label{sec:numerics}

\newlength{\panelwidth}
\setlength{\panelwidth}{0.4\textwidth}
\newlength{\panelheight}
\setlength{\panelheight}{0.3\textwidth}
\begin{figure}[t!]
    \centering
    \includegraphics[width=\panelwidth]{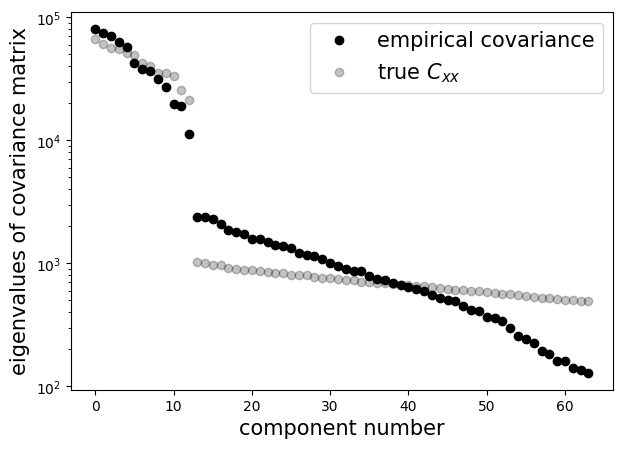}
    \includegraphics[width=\panelwidth]{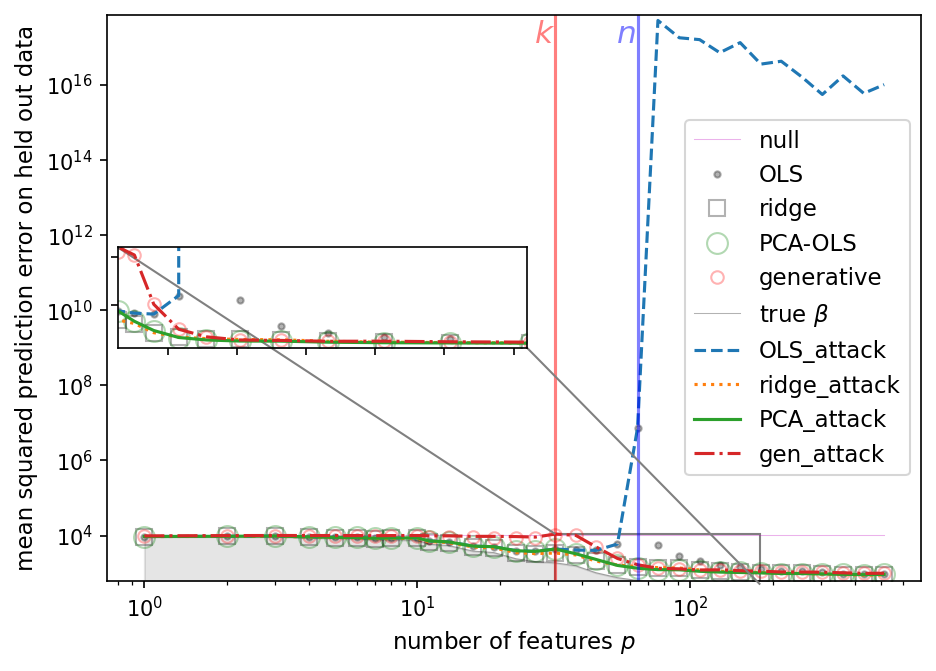}
    \caption{\textbf{(Left)}Eigenvalues of the empirical and model covariance $C_{xx}^p$ for $p=64$, with a gap at component number $k=32$. \textbf{(Right)} MSE for prediction of $y$ on test under data-poisoning attacks of magnitude $\epsilon=1$, with fixed $n=64$ and varying $p$. The attacks in OLS for $p>n$ are arbitrarily successful, while not so effective for other regularized methods.\label{fig:attack}}
    \centering
    \includegraphics[height=\panelheight]{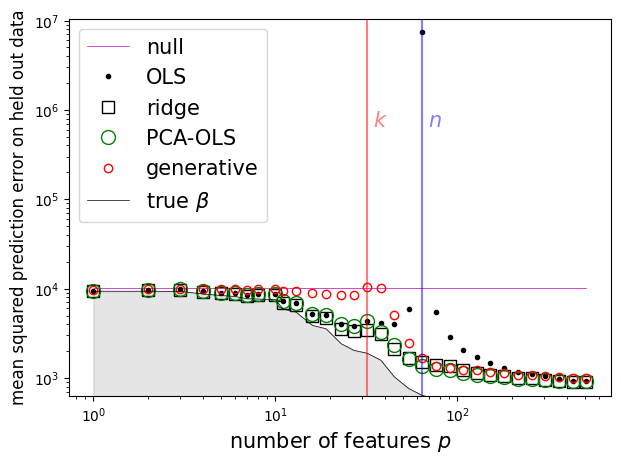}
        \includegraphics[height=\panelheight, trim={1.2cm 0  0 1cm}, clip]{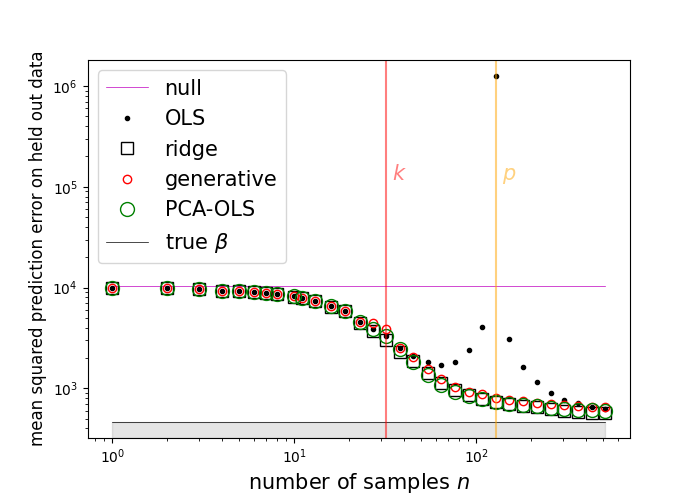}
    \caption{ \textbf{(Left)} MSE for prediction of $y$ on test sets. We consider different regression methods, namely OLS, ridge regression (with optimal ridge parameter found by cross-validation), the generative model described in Section \ref{sec:setup}, and OLS performed on a PCA projection of the data to dimension $k=32$. We also report the performance of the null estimator $\hat \beta =0$, and the true linear coefficient $\beta$ (delimiting the shaded region). In these experiments we fix the number of samples $n=64$, and we vary the number of features $p$ generating the data according to the model described in Section \ref{sec:numerics}. 
    \textbf{(Right)} MSE for prediction of $y$ with fixed number of features $p=128$ and varying the number of samples $n$. We observe that OLS exhibits the peaking phenomenon, whereas the other (regularized) estimators have practically the same, monotonically decreasing risk.
    }
    \label{fig:risk}
\end{figure}

We perform numerical experiments\footnote{Code available at: \url{https://github.com/nhuang37/dimensionality_reduction}} on data generated by the model described in Section \ref{sec:setup}. 

\subsection{Generalization and robustness as a function of \texorpdfstring{$n,p$}{n, p}}\label{subsec:exp1}

We consider two settings, one in which the number of features $p$ is fixed and the number of samples $n$ vary, and another in which the number of samples is fixed and the number of features vary. In order to define the latter we consider a large $(N+1)\times (N+1)$ covariance matrix $\Sigma_N$ for $N=512$ constructed as $W\,W^\top$ where $W$ is a random matrix with standard i.i.d Gaussian entries. We manipulate the eigenvalues of $\Sigma_N$ to produce a gap, where the top $k=32$ eigenvalues are larger than the rest (the manipulation is done by rescaling its top 32 eigenvalues to be 100 times larger).
We define $C_{xx}^p$ to be the first $p\times p$ block of $\Sigma_N$. Since the eigenvectors of $\Sigma_N$ are incoherent with the axis, we obtain that $C_{xx}^p$ exhibits the same structure, as illustrated in Figure \ref{fig:attack} (left).


In Figures \ref{fig:risk} and \ref{fig:attack} (right) we report the mean square prediction error (MSE), defined as 
\begin{equation} \label{mse}
    \operatorname{MSE}(\hat\beta, X_{\text{test}}, Y_\text{test})=\frac{1}{T}\sum_{t=1}^T \sum_{(x^*_i,y^*_i) \in (X^*_t,Y^*_t)} \frac{1}{n_{\text{test}}} \|x^*_i\hat \beta_t - y^*_i \|^2 \, ,
\end{equation}
where $X_{\text{test}}, Y_\text{test}$ denote the test sets with $T$ trials, and each trial $(X^*_t,Y^*_t)$ consists of $n_{\text{test}}$ of data samples. We choose $T=16, n_{\text{test}} = 256$ in our simulations.

For estimators $\hat\beta_{\text{OLS}},\, \hat\beta_{\text{generative}}, \hat\beta_{\lambda}$ (where the ridge optimal regularization parameter $\lambda$ is found with cross-validation), and $\hat\beta_{\text{PCA},k}$ for $k=32$. We compare with the prediction error of the true $\beta$ and with the prediction error of the null estimator $\hat\beta=0$. 

In Figure \ref{fig:risk} we fix number of samples $n=64$ and vary the number of features $p$. We observe that OLS exhibits the peaking phenomenon, whereas the other regularized methods do not. In the overparameterized regime, the risks of all the estimators coincide and consistently decrease with $p$. The generative model fails to predict the data when the number of features is less than $k$. This is expected as the latent space has too much freedom, or the latent variable is too powerful such that the model doesn't utilize the signals from the training labels.

In Figure \ref{fig:risk} we fix the number of features $p=128$ and vary the number of samples. We observe that ordinary least squares exhibit the peaking phenomenon and the \emph{more data can hurt} behavior at around $p\approx n$, while the rest of the regularized estimators perform similarly, with monotonically decreasing mean square prediction error as the number of samples increases.

In Figure \ref{fig:attack} (right) we consider the setting from Figure \ref{fig:risk} (left), and we evaluate the data-poisoning attack described in Section \ref{sec:attack} with $\epsilon=1$. We observe the ordinary least squares estimator is very susceptible to the attack, whereas the regularized methods are robust.

\subsection{Comparison of projection-based methods and the choice of \texorpdfstring{$k$}{k}} \label{subsec:exp2}

We compare the performance of five different projection-based methods as a function of the projection dimension $k$: PCA-OLS; oracle-PCR from \cite{xu2019number}; partial least squares from \cite{cook2019partial}; random Gaussian projection from \cite{Ba2020Generalization}; and orthogonal projection from \cite{lin2020causes} (including both the unregularized case and optimally-tuned ridge regression case).

To understand the interaction between the method and the data generating process, we consider four different population covariance structures, as illustrated in Figure \ref{fig:diff_cov}. For each covariance model, we fix the number of samples to $n=50$ and the number of parameters to $p=75$, and vary the projection dimension $k$. We generate $\beta \sim \mathcal{N}(0,I_p)$, corresponding to the condition in Theorem \ref{thm.risk3}. The performance is evaluated by out of sample mean square error (equation \eqref{mse}) averaged over $T=10$ trials with $n_{\text{test}} = 256$.
For all random projection methods, we also averaged over five experiments with random weights. 

\begin{figure}[t!]
    \centering
     \includegraphics[width=2\panelwidth]{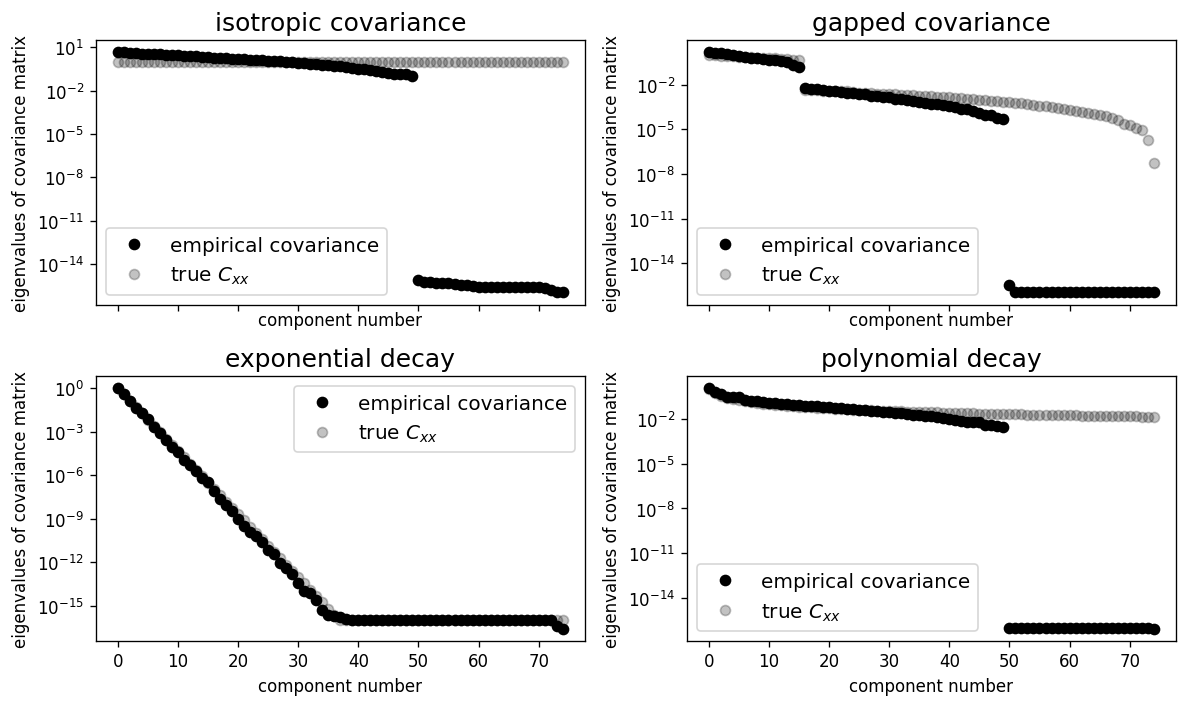}
    \caption{Different covariance structures in our experiments: isotropic covariance refers to $C_{xx}$ being the identity ; gapped covariance refers to a planted eigenvalue gap (at component 16); exponential(polynomial) decay refers to the different decay patterns of the eigenvalues. All the largest eigenvalue is chosen to be $1$.}
    \label{fig:diff_cov}
    \centering
    \includegraphics[width=2\panelwidth]{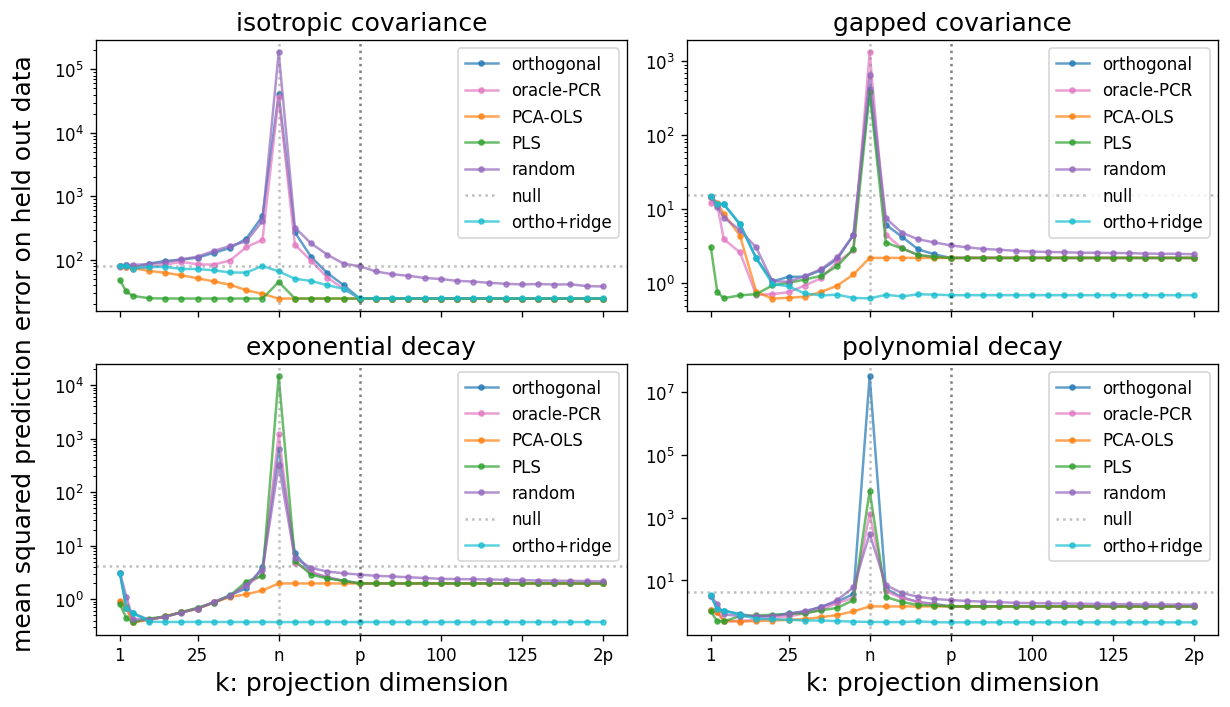}
    \caption{High signal-to-noise (SNR=16) setting with different covariance structure.
    Random orthogonal projection method is labeled as ``orthogonal" for the case without regularization, and ``ortho+ridge" for the case with optimally-tuned ridge regularization (by leave-one-out-cross-validation).
}
    \label{fig:highsnr}
\end{figure}

Figure \ref{fig:highsnr} illustrates the advantage of \textit{data-dependent} dimensionality reduction,
as compared to those overparameterized projection methods that are independent of training data and without regularization, especially when the population covariance exhibits eigenvalue decay. 
Remarkably, oracle-PCR performs worse than PCA-OLS at almost all choices of $k$, as analyzed in Section \ref{subsec:oracle}.
 
Similarly, random Gaussian projection and random orthogonal projection are inferior to PCA-OLS due to the lack of proper regularization. Nevertheless, random orthogonal projection with ridge regularization achieves similar minimum risk as PCA-OLS, where further overparameterization improves its performance for most cases. 



In Supplementary \ref{supp:experiments} we provide further numerical experiments where we investigate different signal-to-noise ratio ($\text{SNR} = \frac{\| \beta \| }{\sigma}$) and the setting where the coefficients of $\beta$ are misaligned with the eigenvalues of $C_{xx}$.
%



\section{Discussion}

Regularization plays an important role in inference. Large-capacity models that can perfectly fit the training data will also generalize well if appropriate regularization is in play~\cite{rougier2020exact}. 
In this article we studied different regression models in the context of the double-descent phenomenon~\cite{belkin2019reconciling}, with different forms of regularization. 
We show that dimensionality reduction is indeed a form of regularization, one that under certain assumptions avoids the ``peaking phenomenon"  (avoids, in the sense that the risk is bounded when the number of features equals the number of samples). Our difference with previous work in \cite{xu2019number} is that our dimensionality reduction is based on the empirical singular values of the features (it is derived from a standard principal components analysis); it is not based on the eigenvalues of the unobservable (true) population covariance.

More precisely, we provide non-asymptotic bounds for the risk of PCA-OLS, which is the ordinary least squares estimator following a projection of the features to their principal components (also known in the literature as PCR or principal component regression). Our main results hold in the overparameterized regime, where the number of samples $n$ is smaller than the number of parameters $p$, and the effective rank of the data is $o(n)$. A future research direction is to analyze PCA-OLS in the asymptotic regime. A particularly interesting question is under what data conditions can we prove that the risk of PCA-OLS decreases with $p$ in the asymptotic regime $\frac{n}{p}=\gamma<1$.
 
Besides analyzing the generalization performance of a particular method, we compare double-descent curves of various projection-based methods. Based on our empirical results, we conjecture that \textit{data-dependent} dimensionality reduction methods are superior to those independent of training data that lack proper regularization. 
Alternatively, we can view different projection methods as feature selection procedure. For example: PCA-OLS chooses features based on the maximum variance direction in the data; oracle-PCR chooses features with the prior from the true data model; random projection selects feature randomly. Our findings are connected with the discussion in  \cite{Belkin_2020} that the double-descent curve depends on the feature selection procedure. Moreover, we show that it is also driven by the data models and regularization.

We also show that unregularized least squares is extremely vulnerable to data-poisoning attacks in the overparameterized regime, whereas other regularized methods are not vulnerable. Our approach has the limitation that the attacks considered were tailored to ordinary least squares and not the regularized methods. We conjecture, however, that it is much harder to achieve arbitrarily large risk when attacking regularized methods.

Many different forms of regularization have similar flavors (for instance $l_2, l_1$ regularization). We propose a generative model that is closely related to PCA-OLS and PLS as shown in Section \ref{sec:setup}. Intuitively, in the overparameterized regime, the generative model learns a low-dimensional latent space that fits the data and labels as closely as possible, while PCA-OLS projects the data into a low-dimensional space that is close to the original data space in terms of the projection error. 

Finally, we emphasize that the motivation behind the generative model comes from the physical sciences, where many models used in practice have the generative structure, for instance \cite{ness2015cannon}. In this work, we have proved that its special case, PCA-OLS, does not exhibit the ``peaking phenomenon". Our empirical results suggest that this may hold generally for generative models. Analysis of these kinds of models is a good direction for future research. 



\section{Acknowledgments}
We thank Joshua Agterberg, Edgar Dobriban, Liliana Forzani, Daniel Hsu, Carey Priebe, Liza Rebrova, Bernhard Sch\"olkopf and Rachel Ward for relevant discussions. We also thank the anonymous reviewers for giving us constructive feedback that helped us improve this manuscript significantly.
SV is partially supported by NSF DMS 2044349, EOARD FA9550-18-1-7007, and the NSF-Simons Research Collaboration on the Mathematical and Scientific Foundations of Deep Learning (MoDL) (NSF DMS 2031985).

\appendix

\section{Proof of Theorem \ref{thm.risk}}
\label{app.proof1}

A key ingredient for the proof of Theorem \ref{thm.risk} is the concentration of the eigenvalues of the empirical covariance matrix around the respective eigenvalues of the population covariance matrix. In particular we use the following uniform concentration result by Koltchinskii and Lounici \cite{koltchinskii2017concentration}.

\begin{theorem9}\label{thm:theorem9}
Let $x_1, \cdots, x_n$ be the i.i.d. samples from a Gaussian distribution with mean $0$ and covariance $C_{xx}$. Let $X$ be the sample matrix where the $i$-th row is given by $x_i$. There exists a constant $c$ such that for any constant $1 < t < n$, with probability at least $1 - e^{-t}$:
\begin{equation}
   \|C_{xx}  - \frac{1}{n} \tr{X} X\|_{\text{op}} \le c \lambda_1 \max \Big\{\sqrt{\frac{r_0(C_{xx})}{n}},\frac{r_0(C_{xx})}{n}, \sqrt{\frac{t}{n}} \Big\}, \label{eqn:op}
\end{equation}
where $\lambda_1$ is the largest eigenvalue of $C_{xx}$, and $r_0(C_{xx}) := \frac{\operatorname{tr}(C_{xx})}{\lambda_1}$ is the effective rank.
\end{theorem9}

We first prove Lemmas \ref{lemma.bias} and \ref{lem:variance}, and Theorem \ref{thm.risk} will follow. 

\begin{lemma}\label{lemma.bias}
Let $E := C_{xx}  - \frac{1}{n} \tr{X} X$. Then
\begin{equation}
    \mathbb B \le \| \beta \|^2 \Big(2  \|E \|_{\text{op}} + \lambda_{k+1}  \Big).
\end{equation}
\end{lemma}

\begin{proof}
Let the spectral decomposition of the empirical covariance matrix be:
\begin{equation}
 \frac{1}{n}\tr{X} X =  \sum_{i=1}^{\min{\{n,p\}}} \tilde{\lambda}_i \tilde{u}_i \tr{\tilde{u}_i}, \quad\quad
 \tr{X_{PCA, k}} X_{PCA, k}  = n \sum_{i=1}^k \tilde{\lambda}_i \tilde{u}_i \tr{\tilde{u}_i}.   \label{eqn:tilde}
\end{equation}
Thus:
\begin{align}
\Pi_{X_{\text{PCA}\perp}} &= I - ( \tr{X_{PCA, k}} X_{PCA, k} )^{\dagger}\tr{X_{PCA, k}} X_{PCA, k} 
= I - \sum_{i={1}}^k \tilde{u}_i \tr{\tilde{u}_i},
\end{align}
where $I$ is the $p \times p$ identity matrix.
The bias term can be bounded by:
\begin{align}
 \mathbb B &=  \operatorname{tr}(\beta^\top\,\Pi_{X_{\text{PCA}\perp}}\,C_{xx}\,\Pi_{X_{\text{PCA}\perp}}\,\beta) \\
  &\le \|\Pi_{X_{\text{PCA}\perp}}\,C_{xx}\,\Pi_{X_{\text{PCA}\perp}} \|_{\text{op}} \, \|\beta \|^2 \label{eqn:simple}  
 \equiv \|\mathbb C \|_{\text{op}} \, \|\beta \|^2,
\end{align}
where \eqref{eqn:simple} follows from Von Neumann's trace inequality. $\| \cdot \|_{\text{op}}$ is the operator-norm for matrices. 
Note that:
\begin{align}
    \mathbb C &= (I - \sum_{i=1}^k \tilde{u}_i \tr{\tilde{u}_i}) \, C_{xx} \, (I - \sum_{i=1}^k \tilde{u}_i \tr{\tilde{u}_i}) \\
    &= (I - \sum_{i=1}^k \tilde{u}_i \tr{\tilde{u}_i})\, (C_{xx} - \sum_{i=1}^k \tilde{\lambda}_i \tilde{u}_i \tr{\tilde{u}_i}) \, (I - \sum_{i=1}^k \tilde{u}_i \tr{\tilde{u}_i}). \label{eqn: cxxperp}
\end{align}
Equation \eqref{eqn: cxxperp} holds because
$(I - \sum_{i=1}^k \tilde{u}_i \tr{\tilde{u}_i})$ and  $\sum_{i=1}^k \tilde{\lambda}_i \tilde{u}_i \tr{\tilde{u}_i}$ are orthogonal.
If $k = p < n$, then $\|I - \sum_{i=1}^k \tilde{u}_i \tr{\tilde{u}_i} \|_{\text{op}} = 0$. So, the bias is trivially 0.
Otherwise $\|I - \sum_{i=1}^k \tilde{u}_i \tr{\tilde{u}_i} \|_{\text{op}} = 1$:
\begin{align}
  \|\mathbb C \|_{\text{op}}  &\le \|C_{xx} - \sum_{i=1}^k \tilde{\lambda}_i \tilde{u}_i \tr{\tilde{u}_i} \|_{\text{op}} 
  = \|C_{xx} -  \sum_{i=1}^n \tilde{\lambda}_i \tilde{u}_i \tr{\tilde{u}_i} + \sum_{i=k+1}^n \tilde{\lambda}_i \tilde{u}_i \tr{\tilde{u}_i} \|_{\text{op}}\\  \label{eqn:triangle}
  &\le \|C_{xx}  - \frac{1}{n} \tr{X} X \|_{\text{op}} + \|   \sum_{i=k+1}^n \tilde{\lambda}_i \tilde{u}_i \tr{\tilde{u}_i} \|_{\text{op}}  \\
  & = \|C_{xx}  - \frac{1}{n} \tr{X} X \|_{\text{op}} + \tilde{\lambda}_{k+1}  \\
  & \le 2  \|C_{xx}  - \frac{1}{n} \tr{X} X \|_{\text{op}} + \lambda_{k+1} \label{eqn:bias},
\end{align}
where the last inequality follows from observing that, for any $1 \le k < n$:
\begin{equation}
  \|C_{xx}  - \frac{1}{n} \tr{X} X \|_{\text{op}} = \max_{i=\{1, \cdots, n\} } \left|\lambda_i - \tilde{\lambda}_i \right| \ge  \left| \lambda_{k+1} - \tilde{\lambda}_{k+1} \right| .  \label{eqn:eigenvalue_unif}
\end{equation}

\end{proof}
Lemma \ref{lemma.bias} shows essentially the same bound that Theorem 4 of \cite{bartlett2020benign}, with an extra term coming from the $k+1$-th eigenvalue of $C_{xx}$ (that in practice should be small for the data models in which PCA is suitable). Note that \eqref{eqn:triangle} can be quite loose. It could be refined by applying the Davis-Kahan
theorem under additional assumptions on the eigenvalue separation. 

\begin{lemma} \label{lem:variance}
Let $E := C_{xx}  - \frac{1}{n} \tr{X} X$. Assume $\| E \|_{\text{op}} < \lambda_k$, then
\begin{equation}
 \frac{\sigma^2}{n} \frac{k \lambda_p}{ \lambda_1 + \|E\|_{\text{op}}} \leq \mathbb V \leq \frac{\sigma^2}{n} \frac{k \lambda_1}{ \lambda_k - \|E\|_{\text{op}}}.  \label{eqn:var_loose}
\end{equation}
\end{lemma}

\begin{proof}
Recall from \eqref{eqn:tilde}, $\tilde{\lambda}_i$ denotes the $i$-th eigenvalue of $(1/n)\,X_\text{PCA}^\top X_\text{PCA}$, for $i=1, \cdots, k$. $\lambda_i$ denotes the $i$-th eigenvalue of $C_{xx}$. Thus, $1/\tilde{\lambda}_{1+k-i}$ is the $i$-th largest eigenvalue of $(\frac{1}{n}\,X_\text{PCA}^\top X_\text{PCA})^\dagger$. By Von Neumann's trace inequality we have:
\begin{equation}
    \operatorname{tr}\left((\frac{1}{n}\,X_\text{PCA}^\top X_\text{PCA})^\dagger\,C_{xx}\right)\leq  \sum_{i=1}^k \frac{\lambda_i}{\tilde{\lambda}_{1+k-i}} \le \frac{1}{\tilde{\lambda}_{k}} \sum_{i=1}^k \lambda_i . 
\end{equation}
By \eqref{eqn:eigenvalue_unif}, $\tilde{\lambda}_k \ge \lambda_k - \| E \|_{\text{op}} > 0$, where the second inequality holds from assumption. So we obtain:
\begin{equation}
    \mathbb V = \frac{\sigma^2}{n} \operatorname{tr}\left((\frac{1}{n}\,X_\text{PCA}^\top X_\text{PCA})^\dagger\,C_{xx}\right) \le \frac{\sigma^2}{n}\frac{k \lambda_1}{\lambda_k - \| E \|_{\text{op}}}.
\end{equation}
 
In order to get a lower bound we use Von Neumann's trace inequality:
\begin{equation}
   \operatorname{tr}\left((\frac{1}{n}\,X_\text{PCA}^\top X_\text{PCA})^\dagger\,C_{xx}\right) \geq  \sum_{i=1}^k \frac{\lambda_{p-i+1}}{\tilde{\lambda}_{1+k-i}} \ge \frac{1}{\tilde{\lambda}_{1}} \sum_{i=1}^k \lambda_{p-i+1}.  \label{eqn:var_lower}
\end{equation}
By \eqref{eqn:eigenvalue_unif}, $\tilde{\lambda}_1 \le \lambda_1 + \| E \|_{\text{op}}$, so we obtain:
\begin{equation}
   \mathbb V = \frac{\sigma^2}{n} \operatorname{tr}\left((\frac{1}{n}\,X_\text{PCA}^\top X_\text{PCA})^\dagger\,C_{xx}\right)   \geq \frac{\sigma^2}{n} \frac{k \lambda_p}{\lambda_1 + \|E\|_{\text{op}}} .
\end{equation}
Combining the upper with the lower bound gives us  the result.

\end{proof}

\begin{proof}[Proof of Theorem \ref{thm.risk}]
From lemma \ref{lemma.bias} and lemma \ref{lem:variance}, we can control the bias and variance simultaneously via $\| E \|_{\text{op}}$. We complete the proof by using the upper bound of $\| E \|_{\text{op}}$ in \eqref{eqn:op}. 

\end{proof}

\section{Proof of Theorem \ref{thm.risk2} \label{app.proof2}}

In order to prove Theorem \ref{thm.risk2} we use the following results from \cite{loukas2017close}:
\begin{cor41}
\label{cor41}
For any weights $w_{i j}$ and real $t>0$ :
\begin{equation}
\mathbf{P}\left(\sum_{i \neq j} w_{i j}\left\langle\widetilde{u}_{i}, u_{j}\right\rangle^{2}>t\right) \leq \sum_{i \neq j} \frac{4 w_{i j} k_{j}^{2}}{n t\left(\lambda_{i}-\lambda_{j}\right)^{2}},   \label{eqn:cor41} 
\end{equation}
where $k_{j}^{2}=\lambda_{j}\left(\lambda_{j}+\operatorname{tr}(C_{xx})\right)$ for data generated from Gaussian distribution, and $w_{i j} \neq 0$ when $\lambda_{i} \neq \lambda_{j}$.
\end{cor41}

Let $P_r = u_r \tr{u}_r, \hat{P}_r = \tilde{u}_r \tr{\tilde{u}_r}$ be the empirical and population eigen-projectors, respectively. We need the following set of concentration results from \cite{koltchinskii2017normal}:

\begin{eqn13}
\begin{equation}
\mathbb{E}\|\hat{P}_{r}-P_{r}\|_{2}^{2}=(1+o(1)) \frac{A_{r}(C_{xx})}{n},    \label{eqn:clt_mean}
\end{equation}
where $A_{r}(C_{xx})=2 \operatorname{tr}\left(P_{r} C_{xx} P_{r}\right) \operatorname{tr}\left(C_{r} C_{xx} C_{r}\right)$ and the operator $C_{r}$ is defined as $C_{r}:=$
$\sum_{s \neq r} \frac{P_{s}}{\mu_{r}-\mu_{s}}$. 
\end{eqn13}

\begin{eqn14}
\begin{equation}
\operatorname{Var}\left(\|\hat{P}_{r}-P_{r}\|_{2}^{2}\right)=(1+o(1)) \frac{B_{r}^{2}(C_{xx})}{n^{2}},   \label{eqn:clt_var}
\end{equation}
where $B_{r}(C_{xx}):=2 \sqrt{2}\left\|P_{r} C_{xx} P_{r}\right\|_{2}\left\|C_{r} C_{xx} C_{r}\right\|_{2}$.
\end{eqn14}

Note that under mild assumption, $A_{r}(C_{xx})$ and $B_r (C_{xx})$ have the same order as the effective rank $r_0 (C_{xx})$. Thus, if $r_0 (C_{xx}) = o(n)$, then the empirical eigen-projectors concentrate on their population counterparts. This is the crucial assumption to the following asymptotic normality result:

\begin{eqn15}
Assume effective rank $r_0(C_{xx})=o(n)$:
\begin{equation}
  \frac{\|\hat{P}_{r}-P_{r}\|_{2}^{2}-\mathbb{E}\|\hat{P}_{r}-P_{r}\|_{2}^{2}}{\operatorname{Var}^{1 / 2}\left(\|\hat{P}_{r}-P_{r}\|_{2}^{2}\right)} \sim  \mathcal{N}(0,1) .
\end{equation}
\end{eqn15}

When stating our concentration results, we often build on the result for the $i$-th eigenspace, and then use an intersection bound to conclude the probability for the leading $k$ eigenspaces:

Let $E_i$ be the $i$-th event. By union bound,
\begin{equation}
\mathbf{P} \left(\bigcup_{i=1}^{k} E_{i}^{c}\right) \leq \sum_{i=1}^{k} \mathbf{P} \left(E_{i}^{c}\right).
\end{equation}
Using De Morgan's Law,
\begin{equation}
  \mathbf{P} \left(\left(\bigcap_{i=1}^{k} E_{i}\right)^{c}\right) \leq \sum_{i=1}^{k}\left[1-\mathbf{P} \left(E_{i}\right)\right] \implies  \mathbf{P}\left( \bigcap_{i=1}^k E_i \right) \ge \sum_{i=1}^k \mathbf{P} (E_i) + 1 - k . \label{eqn:intersect}
\end{equation}

\begin{proof}[Proof of Theorem \ref{thm.risk2}] For the upper bound we write
\begin{align}
    \operatorname{tr}\left((\frac{1}{n}\,X_\text{PCA}^\top X_\text{PCA})^\dagger\,C_{xx}\right) &= 
    \operatorname{tr} \left( (\sum_{i=1}^k \frac{1}{\tilde{\lambda}_i} \tilde{u}_i \tr{\tilde{u}_i}) \, (\sum_{j=1}^p \lambda_j u_j \tr{u}_j)\right) 
    =   \sum_{i=1} ^k \sum_{j=1}^p \frac{\lambda_j}{\tilde{\lambda}_i}  \langle \tilde{u}_i, u_j \rangle^2   \\
    &= \sum_{i=1} ^k \left( \frac{\lambda_i}{\tilde{\lambda}_i}  \langle \tilde{u}_i, u_i \rangle^2 + \sum_{j \neq i, j=1}^p \frac{\lambda_j}{\tilde{\lambda}_i}  \langle \tilde{u}_i, u_j \rangle^2  \right) \label{eqn:var_tight} \\
     & \leq \sum_{i=1} ^k \left( \frac{\lambda_i}{\lambda_i + \|E\|_{\text{op}}}  \langle \tilde{u}_i, u_i \rangle^2 + \sum_{j \neq i, j=1}^p \frac{\lambda_j}{\tilde{\lambda}_i} \langle \tilde{u}_i, u_j \rangle^2  \right),
\end{align}
where the last inequality follows from \eqref{eqn:eigenvalue_unif}. 
Now, using Corollary 4.1 from \cite{loukas2017close} (equation \eqref{eqn:cor41}), for each $i$, the following event
\begin{equation}
\sum_{j \neq i} \frac{\lambda_j}{\tilde{\lambda}_i} \langle \tilde{u}_i, u_j \rangle^2 \le t     
\end{equation}
holds with probability at least $1- \sum_{j \neq i}\frac{4 w_{i j} k_{j}^{2}}{n t\left(\lambda_{i}-\lambda_{j}\right)^{2}}$, where $w_{i j} = \frac{\lambda_j}{\tilde{\lambda}_i}, k_{j}^{2}=\lambda_{j}\left(\lambda_{j}+\operatorname{tr}(C_{xx})\right)$. Then the probability of all $i = 1, \cdots, k$ terms being upper bounded by $t$ is at least $1 - \sum_{i=1}^k \sum_{j \neq i} \frac{4 w_{i j} k_{j}^{2}}{n t\left(\lambda_{i}-\lambda_{j}\right)^{2}}$. Together with the fact that $\langle \tilde{u}_i, u_j \rangle^2 \le 1$, we have
\begin{equation}
     \operatorname{tr}\left((\frac{1}{n}\,X_\text{PCA}^\top X_\text{PCA})^\dagger\,C_{xx}\right)
    \le \sum_{i=1} ^k \left( \frac{\lambda_i}{\lambda_i + \|E\|_{\text{op}}}  + t \right),
\end{equation}
with probability at least $1 - \sum_{i=1}^k \sum_{j \neq i} \frac{4 w_{i j} k_{j}^{2}}{n t\left(\lambda_{i}-\lambda_{j}\right)^{2}}$. Note that this probability is valid when $r_o(C_{xx}) := \frac{\operatorname{tr}(C_{xx})}{\lambda_1} = o(n)$ when $k$ is fixed while $n, p \to \infty$.

For the lower bound we start from equation \eqref{eqn:var_tight} and drop the second term where $j \neq i$:
\begin{align}
   \operatorname{tr}\left((\frac{1}{n}\,X_\text{PCA}^\top X_\text{PCA})^\dagger\,C_{xx}\right) 
   &\ge \sum_{i=1} ^k  \frac{\lambda_i}{\tilde{\lambda}_i}  \langle \tilde{u}_i, u_i \rangle^2 
   \ge \sum_{i=1} ^k  \frac{\lambda_i}{\lambda_i - \|E\|_{\text{op}}}  \langle \tilde{u}_i, u_i \rangle^2. \label{eqn:var_low}
\end{align}
Let $\Phi(a)$ denote the standard normal distribution. Assume the effective rank $r_0 (C_{xx}) = o(n)$ and apply the asymptotic normality result by \cite{koltchinskii2017concentration}:
\begin{align}
\mathbf{P} \left( \frac{\|\hat{P}_{r}-P_{r}\|_{2}^{2}-\mathbb{E}\|\hat{P}_{r}-P_{r}\|_{2}^{2}}{\operatorname{Var}^{1 / 2}\left(\|\hat{P}_{r}-P_{r}\|_{2}^{2}\right)} \le a \right) = \Phi(a).
\end{align}
Now, observe that
\begin{equation}
 \|\hat{P}_r - P_r \|^2_2 = \| \hat{P}_r \|^2_2  +  \|P\|^2_2  - 2 \langle \hat{P}_r, P_r \rangle  = 2 - 2\langle \tilde{u}_i, u_i \rangle^2.
\end{equation}
Thus,  with probability $\Phi(a)$ :
\begin{equation}
    \langle \tilde{u}_i, u_i \rangle^2 \ge 1 - \frac{1}{2} \left(  \mathbb{E}\|\hat{P}_{r}-P_{r}\|_{2}^{2} + a \operatorname{Var}^{1 / 2}\left(\|\hat{P}_{r}-P_{r}\|_{2}^{2}\right)  \right) = 1 - o(1/n),
\end{equation}
where the expectation and the variance are given in equation \eqref{eqn:clt_mean} and \eqref{eqn:clt_var}. Recall that both of them have the same order as $r_0(C_{xx})/ n$. By assumption, the effective rank $r_0(C_{xx})$ is $o(n)$. Thus, both the expectation and the variance grow as $o(n^{-1})$. Note that throughout our proof of Theorem \ref{thm.risk2}, $k$ is fixed, and thus $a$ is some constant.

Plugging back in equation \eqref{eqn:var_low} and combining with an intersection bound from equation \eqref{eqn:intersect}, with probability at least $k \Phi(a) +1 - k$:
\begin{equation}
    \operatorname{tr}\left((\frac{1}{n}\,X_\text{PCA}^\top X_\text{PCA})^\dagger\,C_{xx}\right) 
   \ge \sum_{i=1} ^k  \frac{\lambda_i}{\lambda_i - \|E\|_{\text{op}}} (1 - o(1/n)).
\end{equation}

We remark that as $n \to \infty$, there exists a constant $a$ large enough for the probability to be positive. 
\end{proof}

To summarize, with high probability
\begin{equation}
  \frac{\sigma^2}{n} \sum_{i=1} ^k  \left( \frac{\lambda_i}{\lambda_i - \|E\|_{\text{op}}}  - o(1/n) \right) \leq \mathbb V \leq \frac{\sigma^2}{n} \sum_{i=1} ^k \left( \frac{\lambda_i}{\lambda_i + \|E\|_{\text{op}}}  + t \right) \label{eqn:var_tight_final}.
\end{equation}

\section{Proof of Theorem \ref{thm.risk3}} \label{app.proof3}

\begin{proof}

Assume $\beta$ is randomly drawn from an isotropic distribution: $\mathbb E_{\beta} [\beta] = 0, \, \mathbb E_{\beta} [\beta \, \tr{\beta}] = I$. Then we provide a lower bound by taking expectation over $\beta$:
\begin{align}
  \mathbb E_{\beta} [\mathbb B]  &= \mathbb E_{\beta} \Big[ \operatorname{tr} \left(\beta^\top\,\Pi_{X_{\text{PCA}\perp}}\,C_{xx}\,\Pi_{X_{\text{PCA}\perp}}\,\beta \right) \Big] 
  = \operatorname{tr} \left( \Pi_{X_{\text{PCA}\perp}}\,C_{xx}\,\Pi_{X_{\text{PCA}\perp}}\,\mathbb E_{\beta}[\beta \, \beta^\top] \right) \\
  &= \operatorname{tr} \left( C_{xx}\, \left(I - \Pi_{X_{\text{PCA}}} \right) \right) 
  = \sum_{j=1}^p \lambda_j - \sum_{i=1}^k \sum_{j=1}^p  \lambda_j \langle u_j, \tilde{u}_i \rangle^2  \\
  &= \sum_{j=1}^p \lambda_j -  \sum_{i=1}^k  \lambda_i \langle u_i, \tilde{u}_i \rangle^2  -  \sum_{i=1}^k \sum_{j \neq i }^p \lambda_j \langle u_j, \tilde{u}_i \rangle^2 , \label{eqn:lower_b}
\end{align}
where the last equation follows by splitting the inner products between population and empirical eigenvectors into terms involving $j=i$ (i.e., large) and $j \neq i$ (i.e., small). The large terms can be bounded by $\langle u_i, \tilde{u}_i \rangle^2 \le 1$. The small terms can be controlled using Corollary 4.1 from \cite{loukas2017close} (equation \eqref{eqn:cor41}), which states the following event $E_i$
\begin{equation}
\sum_{j \neq i} \lambda_j \langle \tilde{u}_i, u_j \rangle^2 \le t     
\end{equation}
holds with probability at least $1- \sum_{j \neq i}\frac{4 \lambda_{j} k_{j}^{2}}{n t\left(\lambda_{i}-\lambda_{j}\right)^{2}}$, where $ k_{j}^{2}=\lambda_{j}\left(\lambda_{j}+\operatorname{tr}(C_{xx})\right)$. Then the probability of all $E_i, i = 1, \cdots, k$ being upper bounded by $t$ is at least $1 - \sum_{i=1}^k \sum_{j \neq i} \frac{4 \lambda_{j} k_{j}^{2}}{n t\left(\lambda_{i}-\lambda_{j}\right)^{2}}$.
Thus, with probability at least $1 - \sum_{i=1}^k \sum_{j \neq i} \frac{4 \lambda_{j} k_{j}^{2}}{n t\left(\lambda_{i}-\lambda_{j}\right)^{2}}$,
\begin{equation}
   \mathbb E_{\beta} [\mathbb B]  \geq \sum_{j=1}^p \lambda_j - \sum_{i=1}^k \lambda_i - kt = \sum_{i=k+1}^p \lambda_i - kt.
\end{equation}
\end{proof}

\bibliographystyle{siamplain}
\bibliography{references}

\clearpage
\setcounter{section}{0}
\section*{Supplementary material}
\beginsupplement
\renewcommand\appendixname{Supplementary}

\section{Bias-variance decomposition of the risk} \label{supple:bv-decomp}
In order to analyze the risk of the PCA-OLS estimator $\hat\beta = \hat\beta_{\PCA,k}$ we use the standard decomposition of bias plus variance. In this particular case it is
\begin{align}
    \tilde\beta &= \Pi_{X_{\text{PCA}}}\,\beta \, 
    \\
    \mathcal R(\hat\beta\given X) &= \underbrace{\mathbb E_{x_\ast} [(x_\ast^\top(\beta - \tilde\beta))^2 \given X] }_{\text{bias squared}} 
   + \underbrace{\mathbb E_{Y,x_\ast} [(x_\ast^\top(\hat\beta - \tilde \beta))^2 \given X]}_{\text{variance}} + \sigma^2
    \\
    &= \underbrace{\beta^\top\,\Pi_{X_{\text{PCA}\perp}}\,C_{xx}\,\Pi_{X_{\text{PCA}\perp}}\,\beta}_{\text{bias squared}} + \underbrace{\frac{\sigma^2}{n} \operatorname{tr}\left((\frac{1}{n}\,X_\text{PCA}^\top X_\text{PCA})^\dagger\,C_{xx}\right)}_{\text{variance}}
    + \sigma^2,
\end{align}
where $X_\text{PCA}$, $\Pi_{X_\text{PCA}}$, and $\Pi_{X_{\text{PCA}\perp}}$ are the equivalents of their non-PCA counterparts for the rank-$k$ PCA approximation to $X$.



\begin{proof}
Let $X_p := X_{\PCA,k}$. We have:
\begin{align}
    \hat{\beta} & =  X_p^{\dagger} \, Y , \\
    Y &= X \beta + \epsilon = (X_{p} + X_{p}^{\perp}) \beta + \epsilon ,  \\
     \mathcal R(\hat\beta\given X) &= \mathbb E_{Y, x_\ast} [(x_\ast^\top(\beta - \hat{\beta}))^2 \given X]  + \sigma^2 \\
     &= \mathbb E_{Y, x_\ast} [(x_\ast^\top(\beta - X_p^{\dagger} \, ( X \beta + \epsilon )))^2 \given X]  + \sigma^2 \\
     &= \mathbb E_{Y, x_\ast} [(x_\ast^\top(I - X_p^{\dagger} X) \, \beta - x_\ast^\top  X_p^{\dagger} \, \epsilon )^2 \given X] + \sigma^2  \\
     &=  \mathbb E_{Y, x_\ast} [(x_\ast^\top(I - X_p^{\dagger}  X_{p} ) \, \beta)^2 \given X]  +  \mathbb E_{Y, x_\ast} [ (x_\ast^\top  X_p^{\dagger} \, \epsilon ) )^2 \given X]  + \sigma^2 \label{eqn:vanish}\\
     &= \beta^\top\,\Pi_{X_{p\perp}} \,C_{xx}\,\Pi_{X_{p\perp}} \beta + \operatorname{tr}( {{X_p^{\dagger}}}^\top\,\Sigma\, X_p^{\dagger} \, \mathbb{E}\left[\epsilon \epsilon^{\top} | X\right])  + \sigma^2 \\
     &= \beta^\top\,\Pi_{X_{p\perp}} \,C_{xx}\,\Pi_{X_{p\perp}} \beta  + \sigma^2 \operatorname{tr}( X_p^{\dagger} \, {{X_p^{\dagger}}}^\top \,\Sigma ) + \sigma^2  \label{eqn:iidnoise} \\
     &= \beta^\top\,\Pi_{X_{p\perp}} \,C_{xx}\,\Pi_{X_{p\perp}} \beta   + \frac{\sigma^2}{n} \operatorname{tr}\left((\frac{1}{n}\,X_p^\top X_p)^\dagger\,C_{xx}\right)  + \sigma^2 . \label{eqn:final}
\end{align}

Note that the cross term in \eqref{eqn:vanish} does vanish since $\epsilon$ has zero mean conditioned on $X,$ and is independent of $x^*$. Note that \eqref{eqn:iidnoise} follows from the assumption that noise is i.i.d. 
Finally, \eqref{eqn:final} follows from the fact that:
\begin{equation}
  X_p^{\dagger} \, {X_p^{\dagger}}^\top = {(X_p^\top X_p)}^\dagger,
\end{equation}
which can be obtained by letting $X_p = \sum_{i=1}^k s_i \tilde{v}_i \tr{\tilde{u}}_i$ and observing $X_p^{\dagger} = \sum_{i=1}^k \frac{1}{s_i} \tilde{u}_i \tr{\tilde{v}}_i$.
\end{proof}

\section{Random Gaussian projections on isotropic data}
\label{supple:rand_proj}
Given the data matrix $X$ with rows $x_i\sim \mathcal N(0_p,I_p), \, i=1,\ldots, n$ and a random Gaussian projection matrix $\Pi = [w_1, \cdots, w_k] \in \R^{p \times k} $, where $w_{i} \stackrel{\text { i.i.d. }}{\sim} \mathcal{N}\left(0_p, p^{-1} I_{p}\right)$, \cite{Ba2020Generalization} investigated three cases and concluded the following:
\begin{itemize}
    \item Case 1: $\operatorname{rank}{(X \, \Pi)} = p < \min\{n,k\}$: this effectively reduced to the underparameterized case studied in \cite{hastie2020surprises}. In the limit of $n,p,k \to \infty$, the bias tends to zero and the variance tends to $\frac{p}{n-p} \sigma^2$.
    \item Case 2: $\operatorname{rank}{(X \, \Pi)} = k < n < p$: this is similar to PCA-OLS in the overparameterized regime. The bias is no longer zero, due to the dimensionality reduction, while the variance tends to $\frac{k}{n-k} \sigma^2$.
    \item Case 3: $\operatorname{rank}{(X \, \Pi)} = n < k < p$: in this overparameterized setting, the random projection lifts up the features to a higher-dimensional space. The risk decreases monotonically with $k$.
\end{itemize}

We focus on comparing PCA-OLS with random Gaussian projection in the overparameterized setting where $p > n$ (case 2, 3).

Let $\hat{\beta}_{R}$ be the random Gaussian projection method estimator, which is computed by $\hat{\beta}_R = \Pi (X \Pi)^{\dagger} Y$. Let $\| \beta \|^2 = r^2, \operatorname{Var}{(\epsilon)} = \sigma^2, \gamma_1 = p/n, \gamma_2 = k/n$.

\paragraph{Case 2:} When $k < n$, By Theorem 1 in \cite{Ba2020Generalization},
\begin{equation}
  \mathcal R(\hat{\beta}_{R} \mid X) \to \frac{\gamma_{1}-\gamma_{2}}{\gamma_{1} |\gamma_2 -1|} r^{2}+\frac{\gamma_{2}}{|\gamma_2 -1|} \sigma^{2} = \frac{p-k}{p (1-k/n)} r^2 + \frac{k}{n-k} \sigma^2 . \label{eqn:rand}
\end{equation}

From the bias upper bound of PCA-OLS in isotropic setting (\cite{wahl2019note} section 3.2.1) and our variance bound in Theorem \ref{thm.risk}, with high probability, there exists a constant $C >1$ such that:
\begin{equation}
  \mathcal R(\hat{\beta}_{\text{PCA}} \mid X) \le \frac{p-k}{p} r^2 + C \frac{k}{n} \sigma^2 . \label{eqn:pca-comp}
\end{equation}

Comparing \eqref{eqn:rand} and \eqref{eqn:pca-comp}, PCA-OLS has a smaller bias and potentially a smaller variance than random Gaussian projections. 

\paragraph{Case 3}: If $\gamma_2 \to \infty$, by Theorem 1 in \cite{Ba2020Generalization},
\begin{align}
  \mathcal R(\hat{\beta}_{R} \mid X) &\to   \frac{\gamma_{2} |\gamma_1 - 1|}{\gamma_{1} |\gamma_2 - 1|} r^{2}+\frac{|\gamma_1 - 1| + |\gamma_2 - 1|}{|\gamma_1 - 1| \, |\gamma_2 - 1|} \sigma^{2} \\
  & \stackrel{\gamma_2 \to \infty}{\to} \frac{|\gamma_1 - 1|}{\gamma_1 } r^2 + \frac{1}{ |\gamma_1 - 1|} \sigma^2 \\
  & = \mathcal R(\hat{\beta}_{\text{OLS}} \mid X) .
\end{align}
where the last equality follows from Theorem 2 in \cite{hastie2020surprises}. This shows that the risk of $\hat{\beta}_R$ is equivalent to the risk of OLS estimator. However, $ \mathcal R(\hat{\beta}_{R} \mid X) $ monotonically decreases with $k$ (on both bias and variance term). This shows $\hat{\beta}_R$ is strictly worse than OLS (i.e, PCA-OLS) for all $k$ when $k > n$, regardless of the signal-to-noise ratios.

We remark that there are many variants of random Gaussian projection. For example, \cite{slawski2018principal} showed that under stronger assumptions (i.e. the random projections are Johnson-Lindenstrauss transforms), the performance of random Gaussian projections are of the same order as PCA-OLS. 

To conclude, the random Gaussian projections in \cite{Ba2020Generalization} serve as a theoretical tool to analyze the double-descent risk curve, while in practice, stronger assumption are needed to improve its performance as a preprocessing method for regression. 

\section{Adversarial attacks}

\begin{proof}[Proof of Proposition \ref{prop:attack}]

Under the overparameterized regime where $p > n$, the $\hat{\beta}_{OLS}$ is given by equation \eqref{eqn: OLS-over}. Thus:
\begin{align}
    \hat{\beta}_{poison} &= \tr{\tilde{X}} (\tilde{X} \tr{\tilde{X}})^{-1} \tilde{Y} \\
    &= \tr{\tilde{X}}
    \begin{bmatrix}
    X \tr{X} & X x_0 \\
    \tr{x_0} \tr{X}  &  \tr{x_0} x_0
    \end{bmatrix}^{-1} \tilde{Y} \label{eqn: block} \\
    &=  \tr{\tilde{X}}
    \begin{bmatrix}
    f_1(\frac{1}{h}) & f_2(\frac{1}{h}) \\
    f_3(\frac{1}{h}) & \frac{1}{h}
    \end{bmatrix}
    \tilde{Y},
    \\
    h &:= \tr{x}_0 (I - \tr{X} (X \tr{X})^{-1} X) x_0, \label{eqn:attacker}
\end{align}
where $h$ is the square of the projection of $x_0$ onto the $p$-dimensional space orthogonal to the span of $X$ (i.e., the null space of $\tr{X}$), $f_1, f_2, f_3$ are linear functions in $\frac{1}{h}$. We assume the block matrix in equation \eqref{eqn: block} is invertible and thus $h \ne 0$. 

Now, we choose
\begin{equation}
    x_0 =  \frac{\epsilon}{\|\sum_{i=1}^n \alpha_i v_i + \delta \|} \Big(\sum_{i=1}^n \alpha_i v_i + \delta \Big),
\end{equation}
such that $\|x_0 \| \le \epsilon$. Let $\delta \to 0$, then \eqref{eqn:attacker} becomes 
\begin{equation}
 h =  \tr{x}_0 (I x_0 - \tr{X} (X \tr{X})^{-1} X x_0) \to \tr{x}_0 (x_0 - x_0) = 0   .
\end{equation}

Thus, $h$ is arbitrarily small and the risk grows to infinity, and the attack is immensely successful.

\end{proof}

In practice,  we choose $x_0$ be a (random) linear combination of the columns of $X$ plus a small noise, and then normalize it to have $\norm{x_0} = \epsilon$. Meanwhile, $y_0$ can be chosen arbitrarily as $x_0$ drives the success of the attack. 

\begin{proof}[Proof of Proposition \ref{prop:ols-under}]
In the underparameterized regime ($p < n$), the data-poisoning attack becomes:
\begin{align}
    \hat{\beta}_{\text{poison}} &= (\tr{\tilde{X}} \tilde{X})^{-1} \tr{\tilde{X}} \tilde{Y} \\
    &= (\tr{X} X+  x_0 \tr{x_0})^{-1} \tr{\tilde{X}} \tilde{Y}. 
\end{align}
Here, $\tr{X} X$ is full rank, in contrast to the low rank matrix $X \tr{X}$ in the overparameterized setting. The attack effectively adds a rank-1 matrix $x_0 \tr{x_0}$ with $\norm{x_0}^2 \le \epsilon^2$. If the smallest eigenvalue of $\tr{X} X$ is less than $\epsilon$, then the attack can push the smallest eigenvalue of $\tr{X} X + x_0 \tr{x_0}$ infinitely close to $0$, making the risk of OLS grow arbitrarily. 
\end{proof}

\begin{proof}[Proof (sketch) of Corolary \ref{cor:pca-attack}]
PCA-OLS first projects the features $X \in \R^{n \times p}$ to a low-dimensional space $\R^{n \times k}$, and then perform OLS on a rank-$k$ approximation of the features, $X_{\PCA, k}$. Given $k < \min{\{n,p\}}$, PCA-OLS is effectively in the underparameterized regime. The smallest eigenvalue of $X_{\PCA, k}$ is the $k$-th largest eigenvalue of $\tr{X} X$. Thus, if $\lambda_k (\tr{X} X) \gg \epsilon$, the attack has minimal effect in changing the smallest eigenvalue of $X_{\PCA, k}$, and the risk of PCA-OLS under attack will not deviate much from the original risk.
\end{proof}

\section{Further numerical experiments}\label{supp:experiments}

\begin{figure}[t!]
    \centering
    \includegraphics[width=2\panelwidth]{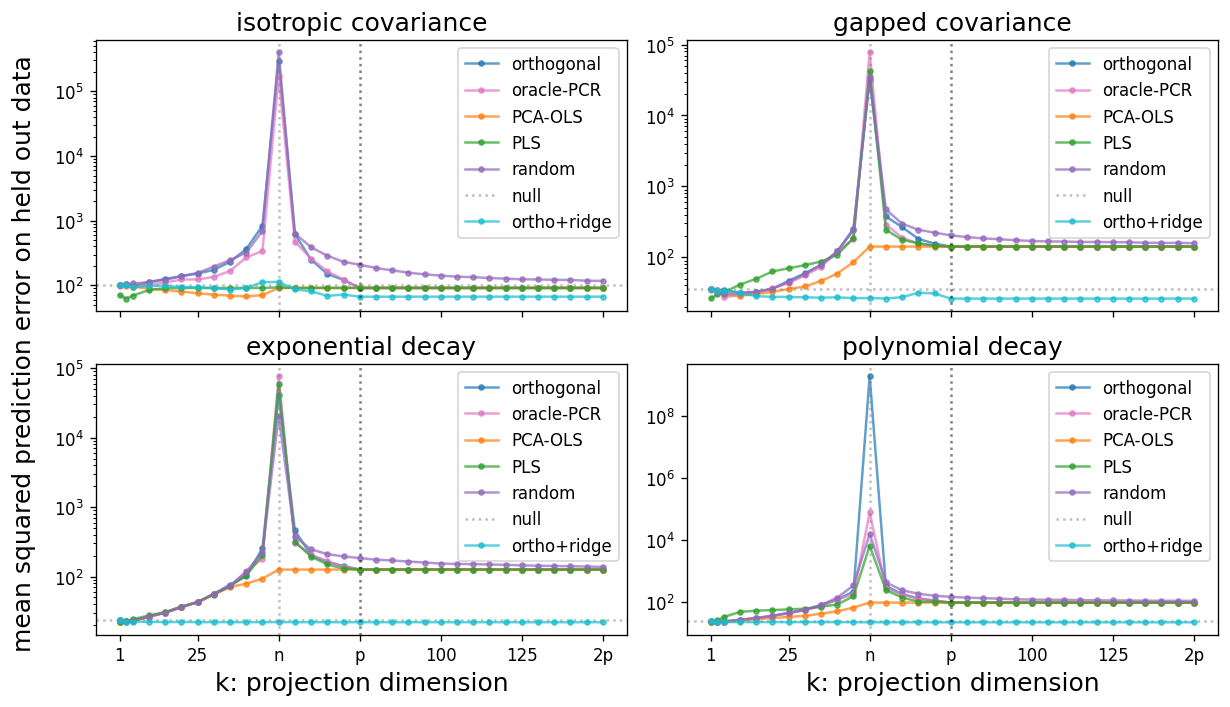}
     \caption{Low signal-to-noise (SNR=2) setting with different covariance structures.}
    \label{fig:lowsnr}
    \centering
    \includegraphics[width=2\panelwidth]{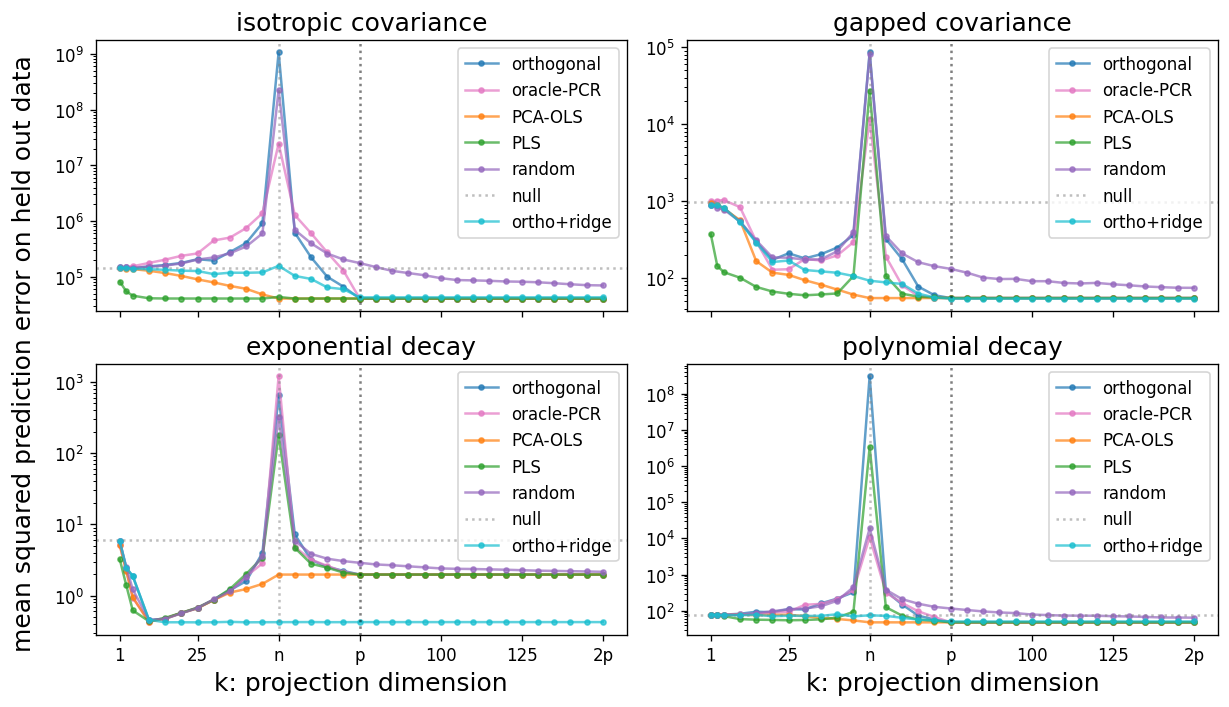}
     \caption{Misalignment setting with the coefficient of $\beta$ is inversely related to the eigenvalues in $C_{xx}$. }
    \label{fig:beta_rev}
    \centering
    \includegraphics[width=2\panelwidth]{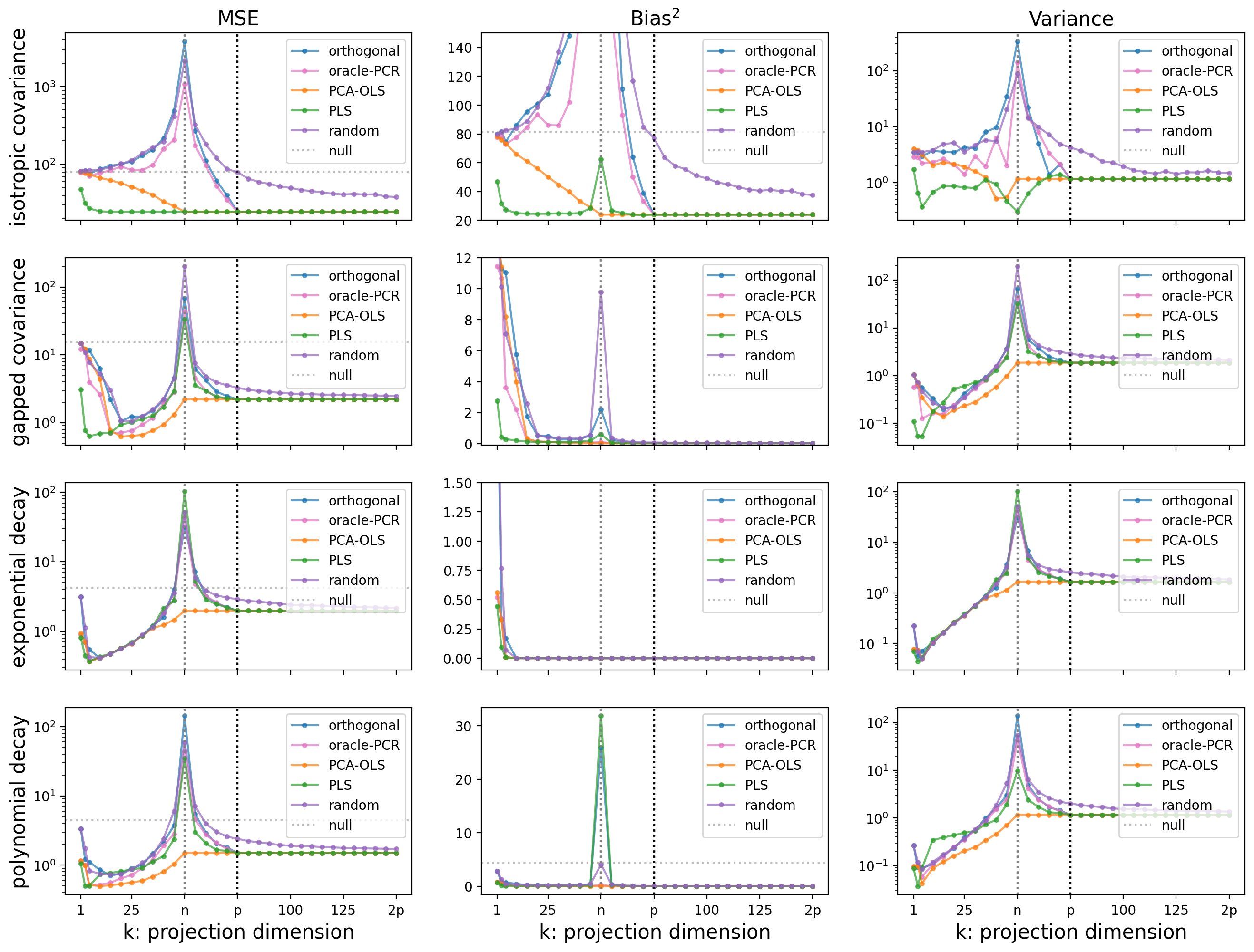}
     \caption{High signal-to-noise (SNR=16) setting with different covariance structures and bias-variance decomposition for selected methods.}
    \label{fig:bvtradeoff}
\end{figure}

In a low SNR setting (Figure \ref{fig:lowsnr}), all the methods perform worse, as expected. In particular, the performance of PLS deteriorates drastically with larger $k$, under the isotropic covariance model and gapped covariance model. This suggests the higher sensitivity of PLS to the signal-to-noise effect. In comparison, PCA-OLS seems to be more robust, as the shape of its risk curves does not change significantly.

In Figure \ref{fig:beta_rev} we study the impact of the alignment of the signal of $\beta$ and the principal components in $C_{xx}$, by letting $\beta = [1, 2, \cdots, p-1, p]$. In other words, large coefficients of $\beta$ are aligned with principal components of $C_{xx}$ with small eigenvalues. In this case, we correct the SNR ratio to keep the same noise variance $\sigma^2$ as in the high SNR setting (shown in Figure \ref{fig:highsnr}). In this misalignment setting, the results for the gapped covariance model changed the most, compared to true $\beta$ that distributes even weights to all PCs (i.e., Figure \ref{fig:highsnr}). As shown in Figure \ref{fig:beta_rev}, PCA-OLS achieves the lowest MSE at $k=n$. A much larger $k$ is needed as the signals concentrate in the principal components with small eigenvalues. Comparatively, PLS reaches the lowest MSE at a smaller $k$ than PCA-OLS. This illuminates the choice of $k$ depends on both the data covariance as well as its alignment with signals on $\beta$.
 
Similar to \cite{wu2020optimal}, we also observe that the risk of oracle-PCR decreases with $k$ in the overparameterized regime, for both the aligned and misaligned settings.

Finally, we analyze the bias and variance terms for different methods in Figure \ref{fig:bvtradeoff} \footnote{Excluding the ridge regularized orthogonal projection method, as the standard bias variance decomposition of OLS does not apply.}. The bias term is computed by:
\begin{equation}
    \mathbb B = \beta^\top\,\Pi_{X_{p\perp}} \,C_{xx}\,\Pi_{X_{p\perp}} \beta  ,
\end{equation}
where $\Pi_{X_{p\perp}} = I -  \Pi (X \Pi)^{\dagger} X $, following the derivation in Supplementary \ref{supple:bv-decomp}. Then we compute the variance by subtracting $ \mathbb B $ and $\sigma^2$ from MSE. Note that this is only an approximation of the true bias and variance component (as the MSE is averaged over Monte-carlo simulations, not the exact risk). As shown in Figure \ref{fig:bvtradeoff}, for most cases: the bias-variance trade-off appears for $k<n$; while both bias and variance monotonically decrease for $k>n$. Note that the bias of PCA-OLS is large with small $k$ for the isotropic covariance model (but even larger for other projection methods except PLS). On the other hand, with eigenvalue decays (row 2-4), PCA-OLS achieves low bias, and does not suffer from high variance.

\end{document}